\documentclass[final]{siamart220329}

\usepackage{amssymb,amsmath,amsfonts,amscd,verbatim}
\usepackage{color,graphicx}
\usepackage{algorithm,algorithmic}
\usepackage{upgreek,comment,url}
\usepackage{mystylefile}
\usepackage{booktabs}
\usepackage{float}
\usepackage{hyperref}
\headers{Stochastic Trace Optimization of Parameter Dependent Matrices}{Arvind K. Saibaba and I. C. F. Ipsen}

\title{Stochastic Trace Optimization of\\
Parameter Dependent Matrices\\
Based on Statistical Learning Theory\thanks{Submitted to the editors DATE.
\funding{The work of both authors 
		was supported in part by NSF grant DMS-1760374. The first author was also supported in part by NSF grants DMS-1845406 and  DMS-2411198.
		The second author was also supported in part by NSF grant CCF-2209510, and DOE grant DE-SC0022085.}}}

\author{Arvind K.\ Saibaba\thanks{Department of Mathematics, North Carolina State University,
NC 27695-8205, USA
  (\email{asaibab@ncsu.edu}, 
  \url{https://asaibab.math.ncsu.edu/})
  .}
\and Ilse C.F.\ Ipsen\thanks{Department of Mathematics, North Carolina State University,
NC 27695-8205, USA
  (\email{ipsen@ncsu.edu}, \url{https://ipsen.math.ncsu.edu/}).}}

\usepackage{amsopn}

\begin{document}
\maketitle
\begin{abstract}
We consider matrices $\ma(\vtheta)\in\rmm$ that depend, possibly nonlinearly, on a parameter $\vtheta$
from a compact parameter space $\Theta$. We present a 
Monte Carlo estimator for minimizing $\trace(\ma(\vtheta))$ over all $\vtheta\in\Theta$, and determine the sampling amount so that the 
backward error of the estimator is bounded with high probability. 
We derive two types of bounds, based on epsilon nets 
and on generic chaining. Both types predict a small 
sampling amount for 
matrices $\ma(\vtheta)$ with small offdiagonal mass, and parameter spaces $\Theta$ of small ``size.''
Dependence on the matrix dimension~$m$ is only
weak or not explicit.
The bounds based on epsilon nets are easier to evaluate and come with fully specified constants. In contrast, the bounds based on chaining depend on the Talagrand functionals which are difficult to evaluate, except in very special cases. Comparisons between the two types of bounds are difficult, although the literature suggests that chaining bounds can be superior. 
\end{abstract}



\section{Introduction}
The goal is to optimize the trace of parameter dependent
matrices.
We consider symmetric matrices $\ma(\vtheta)\in\rmm$
and regularization terms $\mathcal{R}(\vtheta)$
that depend, possibly nonlinearly, on parameters $\vtheta$ from a compact
parameter space $\Theta \subset \real^K$. The objective
is to minimize, over all elements of $\Theta$, 
\begin{equation}\label{eqn:standardform}
    \min_{\vtheta \in \Theta} \{F(\vtheta) + \mathcal{R}(\vtheta)\} \qquad \text{where\quad } F(\vtheta) \equiv \trace(\ma(\vtheta))  .
\end{equation}
Optimization problems of this form arise in optimal experimental design~\cite{alexanderian2018efficient}, Gaussian processes~\cite{Anitescu2012,dong2017scalable}, biclustered matrix completion~\cite{chi2019going}, and
Bayesian inverse problems~\cite{chung2024}
discussed in Section~\ref{s_appl}. 

In the above applications, it is prohibitively 
expense to compute the matrix $\ma(\vtheta)$ explicitly, while the 
regularization term $\mathcal{R}(\vtheta)$ does not contribute a significant computational cost.  This means, the computation of the trace is a major bottleneck in large-scale implementations. A popular  approximation of the trace relies on the computation
of an optimizer $\hat\vtheta$ by means of a Monte Carlo estimator
\begin{align*}
\min_{\vtheta\in\Theta}\{\hat{F}(\vtheta) + \mathcal{R}(\vtheta)\} \qquad \hat{F}(\vtheta) \equiv \frac{1}{N}\sum_{i=1}^N{\vomega_i\t\ma(\vtheta)\vomega_i}.
\end{align*}
Here 
$\vomega_i\in\real^m$ are independent random vectors satisfying $\mathbb{E} [\vomega_i] = \vzero$ 
and\footnote{Here 
$\mi_m\in\rmm$ denotes the $m\times m$ identity matrix.}
$\mathbb{E}[\vomega_i\vomega_i\t] = \mi_m$, so that $\hat{F}(\vtheta)$ is an unbiased estimator of $F(\vtheta)$. This is an example of stochastic average approximation (SAA) optimization~\cite{shapiro2021lectures}. SAA is computationally attractive, 
since any appropriate deterministic optimization routine can be applied to $\hat{F}(\vtheta)$, and the random vectors $\vomega_i$ can be reused in different 
iterations.

The statistical properties of SAA estimators~$\hat\vtheta$ in the asymptotic regime as $N\rightarrow \infty$  have been studied in great detail~\cite{shapiro2021lectures}. In practice, one can only draw a finite number of samples $N$, hence
it is important to understand the accuracy and the statistical properties of SAA estimators in the non-asymptotic setting. Our goal is to quantify the probabilistic backward error  $F(\hat\vtheta) + \mathcal{R}(\hat\vtheta) - (F(\vtheta^*) + \mathcal{R}(\vtheta^*))$, where $\vtheta^*$ is a minimizer of~\eqref{eqn:standardform}.

There is a lot of existing work on stochastic methods for estimating the trace of a single, fixed matrix: \cite{avron2011randomized,cortinovis2022randomized,Girard1989,Hutch1989,roosta2015improved} just to name a few. The methods make use of concentration inequalities to derive the minimal sampling amount required for a 
desired relative error with high probability. However, they are only capable of estimating the accuracy of the objective function at a single point $\vtheta \in \Theta$, and they do not give information about the accuracy of the SAA estimator, or the accuracy of the objective function in the optimization trajectory.
In contrast, we derive
non-asymptotic bounds on the backward error with
tools from high-dimensional probability and statistical learning theory.

\subsection{Related work}
To our knowledge, ours are the first non-asymptotic SAA estimators for trace optimization.

The stochastic trace estimators for parameter dependent
matrices in \cite{MHKL2025} assume that the parameter
$\vtheta$ is a scalar and that the elements
of~$\ma(\vtheta)$ depend continuously on $\vtheta$.
Since the problem in \cite{MHKL2025} is not an optimization problem, the
same Gaussian random vectors are re-used for different
values of $\vtheta$.
Because these stochastic trace estimators rely on quadratic forms with random vectors, they represent an instance of second-order chaos
\cite{boucheron2003concentration,Krahmer2014suprema,talagrand2005generic}.

\subsection{Contributions}\smallskip
We present probabilistic bounds on the sampling amount~$N$ for 
Hutchinson's trace estimator with Rademacher vectors.
\begin{itemize}
\item This is the one of the first investigations
into the potential of statistical 
learning theory for numerical linear algebra
problems. 
\item In contrast to existing asymptotic analyses of SAA estimators in \cite[Chapter 5]{shapiro2021lectures} and \cite[Section 2.2]{Anitescu2012},
our approach produces 
quantitative and computable error bounds for a finite
amount of samples.

\item In contrast to spectral density estimation methods \cite{MHKL2025}, we make no assumptions on how the
matrices $\ma$ depend on the parameter $\vtheta$. 

\item We derive  bounds from epsilon nets (Section~\ref{s_main}), 
and from generic chaining (Section~\ref{s_chaining}), 
and express them in
terms of subgaussian tails, and also in terms of mixed
tails. All bounds show that the
sampling amount~$N$ is small, 
if the maximal offdiagonal
mass of the matrices $\ma(\vtheta)$ is small and if the
"size" (in  the appropriate measure) of~$\Theta$ is small.
\end{itemize}

\subsection{Main results}
\begin{itemize}
\item For finite parameter spaces $\Theta$, the
sampling amount~$N$ does not depend on the matrix
dimension~$m$, if the offdiagonal mass of the matrices $\ma(\vtheta)$ is small and the
cardinality of $\Theta$ is small (Theorem~\ref{t_4}). There is no dependence 
on the matrix dimension~$m$. 

\item For infinite parameter spaces $\Theta$, 
the bound  on 
the sampling amount~$N$ is the same as the bound
for finite spaces, except that the cardinality 
of $\Theta$ is replaced by the covering number
of the epsilon net which features a weak dependence
on the matrix dimension $m$ through the logarithm
(Theorem~\ref{t_1}).

\item The bounds from chaining 
(Theorems \ref{thm:subgaussrisk} and~\ref{thm:chainingmixed})
do not require the Lipschitz continuity  of $\ma(\vtheta)$ but have the disadvantage of depending on 
Talagrand functionals, which are difficult to evaluate.

\item For spherical parameter spaces $\Theta$ the bounds
on the sampling amount~$N$
also 
depend on the dimension and the radius of $\Theta$,
however they feature unknown constants
(Corollaries \ref{c_54a} and~\ref{c_54b}).
\item 
Although the literature suggests that chaining bounds can be superior, we found 
comparisons of chaining and epsilon net bounds difficult
(Remarks \ref{r_com2} and~\ref{r_com}).
The epsilon nets bounds are easier to evaluate and come with fully specified constants. In contrast, the 
chaining bounds depend on the Talagrand functionals which are difficult to evaluate, except in very special cases. 
\end{itemize}

\subsection{Overview}
After presenting two motivating applications in Bayesian
inverse problems (Section~\ref{s_appl}), and a review of
concentration inequalities and stochastic trace estimators (Section~\ref{s_review}), we derive bounds 
based on epsilon nets (Section~\ref{s_main}) and based on generic
chaining (Section~\ref{s_chaining}), and end with  concluding remarks and avenues for future research (Section~\ref{sec:conc}).

\section{Application to Bayesian inverse problems}\label{s_appl}
We illustrate how Bayesian inverse problems give rise to trace optimization problems of the form~(\ref{eqn:standardform}). 
After reviewing the basics of Bayesian inverse problems, 
we discuss specific applications in 
\textit{Optimal Experimental Design} (Section~\ref{ssec:oed}) and \textit{Hyperparameter  estimation} (Section~\ref{s_hyper}).

Given a parameter vector $\vm \in \real^n$ and a
forward operator $\mf \in \real^{m\times n}$,
the \textit{data acquisition} or 
\textit{forward problem} consists of 
computing a vector $\vd \in \real^m$ so that
\begin{align}\label{e_problem}
\vd = \mf\vm + \veps, \qquad \text{where}\qquad
\veps \sim \mathcal{N}(\vzero, \sigma^2 \mi_m). 
\end{align}
The measurement noise $\veps \in \real^m$ is normally distributed with zero mean and covariance~$\sigma^2 \mi_m$. 
On the other hand,
given a data vector $\vd \in \real^m$ and a
forward operator $\mf \in \real^{m\times n}$,
the \textit{inverse problem} consists of recovering the parameter vector $\vm$. 
In many applications, (\ref{e_problem}) is underdetermined with  $m \le n$, and, thus, ill-posed. 

The Bayesian approach to inverse problems puts a prior distribution
on $\vm$, where the prior is a Gaussian distribution
$\vm \sim \mathcal{N}(\vm_{\rm pr}, \mgam_{\rm pr})$. 
Application of Bayes rule gives the posterior distribution $\vm | \vd \sim \mathcal{N}(\vm_{\rm post}, \mgam_{\rm post})$, where 
\[ \mgam_{\rm post} \equiv (\sigma^{-2} \mf\t\mf + \mgam_{\rm pr}^{-1})^{-1} \qquad 
\vm_{\rm post} \equiv \mgam_{\rm post}(\sigma^{-2}\mf\t\vd + \mgam_{\rm pr}^{-1}\vm_{\rm pr}). \]
The following two Bayesian optimization problems can be cast in the form~\eqref{eqn:standardform}.

\subsection{Optimal Experimental Design}\label{ssec:oed}
Data are important in making inference about parameters, but 
the data acquisition process can be laborious and expensive. \textit{Optimal Experimental Design (OED)} seeks to optimize this process subject to budgetary constraints. In the OED formulation of \textit{sensor placement}, one seeks an optimal placement of $k$ sensors at $m>k$ candidate locations. 

Bayesian D-optimality places the sensors to maximize the expected information gain, which is defined as the expected Kullback-Leibler divergence from the prior to the posterior distribution. The D-optimality problem is of the form
\[ \max_{\mtheta \in \Theta}\left\{ \log\det(\mi + \mg \mtheta \mg\t)\right\},  \]
where $\mg \equiv \sigma^{-1} \mgam^{1/2}_{\rm pr}\mf\t$,
and
$\mtheta \equiv \text{diag}\begin{pmatrix}\theta_1& \cdots &\theta_m\end{pmatrix}$ 
is a diagonal matrix with diagonal elements from
the set 
\[ \Theta = \left\{ \vtheta \in \{0,1\}^m\, \vert \>\sum_{j=1}^m \theta_i = k \right\}. \]

With the identity $\log\det(\mk) = \trace(\log(\mk))$ for nonsingular matrices $\mk$, the D-optimal design problem can be cast as~\eqref{eqn:standardform} with 
\begin{align*}
\ma(\vtheta) = - \log(\mi + \mg \mtheta \mg\t),\qquad
\mathcal{R}(\vtheta) = 0.
\end{align*}
In the context of sensor placement,  $\theta_i$ is $1$ or $0$,  meaning a sensor should either be placed at
or else excluded from the $i$th candidate location, $1 \le i \le m$.

\subsection{Hyperparameter estimation}\label{s_hyper}
In many applications \cite{chung2024}, the prior distribution is expressed in terms of unknown hyperparameters $\vtheta \in \real^K$ that must be estimated from the data. 

The \textit{hierarchical Bayesian} approach treats the hyperparameters as random variables and posits a hyperprior $\vtheta \sim \pi(\vtheta)$, to model the parameter vector  
as $\vm|\vtheta \sim \mathcal{N}(\vm_{\rm pr}(\vtheta), \mgam_{\rm pr}(\vtheta))$.
Application of Bayes rule gives the joint posterior distribution 
\[ \pi_{\rm post}(\vm,\vtheta|\vd) \propto \exp\left( -\frac{1}{2\sigma^2} \|\vd-\mf\vm\|_2^2 -\frac12 \|\vm - \vm_{\rm pr}(\vtheta)\|_{\mgam_{\rm pr}(\vtheta)^{-1}}^2   \right) \pi(\vtheta),\] 
where $\propto$ denotes proportionality up to a constant. This 
joint  distribution is non-Gaussian, and the joint estimation of $\vm$ and $\vtheta$ is, in general, computationally expensive. 

As an alternative, the \textit{empirical Bayesian approach} determines the posterior distribution $\pi(\vtheta|\vd)$ by marginalizing over $\vm$,
\[ \pi(\vtheta|\vd) \propto \frac{1}{\sqrt{\det{\mpsi(\vtheta)}}}\exp\left(-\frac12\| \vd -  \mf \vm_{\rm pr}(\vtheta) \|_{\mpsi(\vtheta)^{-1}}^2 \right) \pi(\vtheta),\]
where $\mpsi(\vtheta) \equiv \mf \mgam_{\rm pr}(\vtheta) \mf\t + \sigma^2 \mi_m$, and computes the maximum a posteriori estimate  by minimizing the negative log-likelihood, 
\begin{align}\label{e_appl1}
\min_{\vtheta \in \real^K} \frac12\| \vd -  \mf \vm_{\rm pr}(\vtheta)\|_{\rm \mpsi(\vtheta)^{-1}}^2 + \frac12 \log\det\mpsi(\vtheta) - \ln \pi(\vtheta). 
\end{align}

As in Section~\ref{ssec:oed}, the identity $\log\det(\mpsi(\vtheta)) = \trace(\log (\mpsi(\vtheta)))$
allows this problem to be cast as~\eqref{eqn:standardform} with 
\begin{align*}
\ma(\vtheta) =\log( \mpsi(\vtheta)), \qquad \mathcal{R}(\vtheta) = \| \vd -  \mf \vm_{\rm pr}(\vtheta)\|_{\rm \mpsi(\vtheta)^{-1}}^2 - 2\ln \pi(\vtheta).
\end{align*}
Although the problem ranges over $\real^K$, we can restrict it to a compact set $\Theta \subset \real^K$. 
The detailed numerical computation of~(\ref{e_appl1})
is discussed in~\cite{chung2024}.
\section{Concentration inequalities and trace 
estimators}\label{s_review}
We present two concentration inequalities (Section~\ref{s_conc}) and properties of Hutchinson's
trace estimator (Section~\ref{s_hutch}).

\subsection{Concentration inequalities}\label{s_conc}
We present two different concentration inequalities that
lead to two different trace bounds, and ultimately two different sampling amounts for trace optimization.

Hoeffding's inequality below is the basis for the
trace bound in Lemma~\ref{l_1}.

\begin{theorem}[Theorem 2.2.6 in 
\cite{vershynin2018high}]\label{t_hoeff}
Let $X_1, \ldots, X_N$ be independent scalar random variables with $a_i\leq X_i\leq b_i$
for $0\leq a_i\leq b_i$, $1\leq i\leq N$. Let 
$\hat{X}\equiv \frac{1}{N}\sum_{i=1}^N{X_i}$. Then for any $t>0$
\begin{align*}
\myP\left[|\hat{X}-\E[\hat{X}]|\geq t\right]\leq 
\exp\left(-\frac{2N^2t^2}{\sum_{i=1}^N{(b_i-a_i)^2}}\right).
\end{align*}
\end{theorem}

The inequality below is the basis for the trace bound in Lemma~\ref{l_2}.

\begin{theorem}[Theorem 2 in \cite{cortinovis2022randomized}]\label{t_CK}
Let $\mb\in\rmm$ be symmetric with $b_{jj}=0$, $1\leq j\leq m$, and $\vomega\in\real^m$
be a Rademacher vector, then
\begin{align*}
\myP\left[|\vomega\t\mb\vomega|\geq t\right]\leq 2\,
\exp\left(-\frac{t^2}{8(\|\mb\|_F^2+t\|\mb\|_2)}\right).
\end{align*}
\end{theorem}

\subsection{Hutchinson's trace estimator}\label{s_hutch}
We review properties and bounds for Hutchinson's trace estimator with Rademacher vectors \cite{Girard1989,Hutch1989}.

A Rademacher vector $\vomega\in\real^m$  has independent elements $\omega_j=\pm 1$, each chosen with probability $1/2$, $1\leq j\leq m$.

Given $\ma\in\rmm$ and a Rademacher vector $\vomega\in\real^m$, the quadratic form
$\vomega\t\ma\vomega$ is an unbiased 
estimator \cite[Proposition~1]{Hutch1989},
that is,
\begin{align*}
\E[\vomega\t\ma\vomega]=\trace(\ma) =\sum_{i=1}^m{\ve_i\t\ma\ve_i}.
\end{align*}
Extract the offdiagonal part of $\ma$,
\begin{align}\label{e_abar}
\overline{\ma}\equiv \ma-\diag(\ma).
\end{align}
If $\ma$ is also symmetric, then the variance can be expressed as
\begin{align*}
\V[\vomega\t\ma\vomega]=
2 \|\overline{\ma}\|_F^2,
\end{align*}
because  the elements of $\vomega$ satisfy $\omega_j^2=1$.
Furthermore, since all elements of $\vomega$ are bounded by one in magnitude,
 the random variable
$\vomega\t \ma\vomega$ is bounded by
\begin{align}\label{e_hutch2}
\trace(\ma)-\|\overline{\ma}\|_M\leq \vomega\t \ma\vomega\leq
\trace(\ma)+\|\overline{\ma}\|_M,
\end{align}
where for any matrix $\mb\in\rmm$ we define the norm
\begin{align}\label{e_Mnorm}
\|\mb\|_M\equiv 
\sum_{i=1}^m{\sum_{j=1}^m{|b_{ij}|}}=\|\myvec(\mb)\|_1.
\end{align}

For $N\geq 1$ independent Rademacher vectors $\vomega_i\in\real^m$, the
Monte Carlo trace estimator
\begin{align}\label{e_hutch3}
\hat{F}\equiv\frac{1}{N}\sum_{i=1}^N{\vomega_i\t\ma\vomega_i}
\end{align}
is an unbiased estimator with 
\begin{align}\label{e_hutch1}
\E[\hat{F}]=\trace(\ma),\qquad \V[\hat{F}]=\frac{1}{N}\V[\vomega_i\t\ma\vomega_i]=
\frac{2}{N} \|\overline{\ma}\|_F^2,
\end{align}
since the vectors $\vomega_i$ are independent.

The deviation from expectation of Hutchinson's estimator can be bounded for
general matrices $\ma$ by means of Hoeffding's inequality.
Tail bounds  of the form below are called subgaussian \cite[Section 2.5]{vershynin2018high}.

\begin{lemma}\label{l_1}
Let $\ma\in\rmm$, $\overline{\ma}$ as defined in (\ref{e_abar}), $\hat{F}$ as defined 
in~(\ref{e_hutch3}), and $t>0$. Then
\begin{align*}
\myP\left[|\hat{F}-\trace(\ma)|\geq t\right]\leq 
\exp\left(\frac{-Nt^2}{2\,\|\overline{\ma}\|_M^2}\right),
\end{align*}
where $\|\overline{\ma}\|_F\leq \|\overline{\ma}\|_M\leq m\|\overline{\ma}\|_F$.
\end{lemma}

\begin{proof} Set $X_i\equiv \vomega_i\t\ma\boldsymbol\vomega_i$, 
$1\leq i\leq N$. From (\ref{e_hutch2}) follows that  $a_i\leq X_i \leq b_i$ with 
\begin{align*}
a_i=\trace(\ma)-\|\overline{\ma}\|_M,\qquad
b_i=\trace(\ma)+\|\overline{\ma}\|_M,\qquad 1\leq i \leq N.
\end{align*}
Then $\sum_{i=1}^N{(b_i-a_i)^2}=4N\|\overline{\ma}\|_M^2.$ Now apply Theorem~\ref{t_hoeff}.

For the bounds on $\|\overline{\ma}\|_M$, apply the vector norm relations to 
$\|\overline{\ma}\|_F=\|\myvec(\ma)\|_2$ and
$\|\overline{\ma}\|_M=\|\myvec(\ma)\|_1$, and use the fact that $\overline{\ma}$ has
at most $m(m-1)\leq m^2$ non-zero elements.
\end{proof}

Lemma~\ref{l_1} implies that the probability of
$\hat{F}$ being an accurate estimator increases
with decreasing 
offdiagonal mass $\|\overline{\ma}\|_M$ of $\ma$.

The bound below is an alternative to Lemma~\ref{l_1}. It
requires a symmetric matrix, although a real nonsymmetric matrix
$\mb$ can be replaced by its symmetric part, 
$\vx\t\mb\vx=\frac{1}{2}\vx\t(\mb+\mb\t)\vx$ 
\cite[Section~1]{cortinovis2022randomized}.
Tail bounds of the form below are referred to as
mixed tails since they represent a mixture of subgaussian and subexponential tails \cite[Section 2.8]{vershynin2018high}.

\begin{lemma}[Corollary~1 in \cite{cortinovis2022randomized}]\label{l_2}
Let  $\ma\in\rmm$ be symmetric, $\overline{\ma}$ as defined 
in~(\ref{e_abar}), $\hat{F}$ as defined in~(\ref{e_hutch3}) and $t>0$. Then
\begin{align*}
\myP\left[|\hat{F}-\trace(\ma)|\geq t\right]\leq 2 \exp\left(\frac{-Nt^2}{8(\|\overline{\ma}\|_F^2+t\|\overline{\ma}\|_2)}\right).
\end{align*}
\end{lemma}

\begin{comment}
\begin{proof}
As in \cite[Proof of Theorem 1]{cortinovis2022randomized}, set
$\mb=\frac{1}{N}\mi_{N}\bigotimes \overline{\ma}\in\real^{(mN)\times (mN)}$, and let 
$\vomega\in\real^{mN}$ be a Rademacher vector with
$\vomega^T\equiv \begin{bmatrix}\vomega_1^T& \cdots&\vomega_N\end{bmatrix}^T$.
Then
\begin{align*}
\vomega^T\mb\vomega=\frac{1}{N}\sum_{i=1}^{N}{\vomega_i^T\overline{\ma}\vomega_i}=\frac{1}{N}\sum_{i=1}^N{\left(\vomega_i^T\ma\vomega_i-\trace(\ma)\right)}
\end{align*}
with $\|\mb\|_F=\frac{1}{\sqrt{N}}\|\overline{\ma}\|_F$ and 
$\|\mb\|_2=\frac{1}{N}\|\overline{\ma}\|_2$.
Now apply Theorem~\ref{t_CK}.
\end{proof}
\end{comment}

Like Lemma~\ref{l_1}, Lemma~\ref{l_2} implies that the probability of
$\hat{F}$ being an accurate estimator, for a fixed $t >0$,  increases
with decreasing 
offdiagonal mass of $\ma$.

\section{Bounds based on nets}\label{s_main}
After listing the assumptions (Section~\ref{s_ass}),
we present probabilistic bounds 
on the backward error for matrices with finite parameter spaces (Section~\ref{s_finite}) 
and infinite parameter spaces (Section~\ref{s_infinite});
and end with auxiliary results 
for the proofs of Theorem~\ref{t_1}
and Corollary~\ref{c_1}
(Section~\ref{s_auxiliary}).

\subsection{Assumptions}\label{s_ass}
Let $\ma(\vtheta)\in\rmm$ be a set of symmetric matrices that depend on parameters $\vtheta$ from a 
compact parameter space $\Theta \subset \real^K$.

We use the following abbreviations that focus on the
parameters~$\vtheta$.
The trace and its desired minimizer $\vtheta^*$ are
given by
\begin{align}\label{e_theta2}
F(\vtheta)\equiv \trace(\ma(\vtheta)),\qquad
F(\vtheta^*)\equiv\min_{\vtheta\in\Theta}{F(\vtheta)}.
\end{align}
In practice, we can only compute an estimator and its minimizer $\hat{\vtheta}$, based 
on $N\geq 1$ independent Rademacher vectors
$\vomega_i\in\real^m$, 
given by
\begin{align*}
\hat{F}(\vtheta)\equiv \frac{1}{N}\sum_{i=1}^N{\vomega_i\t\ma(\vtheta)\vomega_i},\qquad
\hat{F}(\hat{\vtheta})\equiv\min_{\vtheta\in\Theta}{\hat{F}(\vtheta)}.
\end{align*}

From (\ref{e_hutch1}) follows that 
$\hat{F}(\vtheta)$ is an unbiased estimator,
\begin{align}\label{e_unbiased}
\E[\hat{F}(\vtheta)]=F(\vtheta)\qquad \text{for all}\  \vtheta\in\Theta.
\end{align}

\begin{remark}\label{r_R}
    In a number of applications, the optimization problem~\eqref{eqn:standardform} contains a regularization term $\mathcal{R}(\vtheta)$. Since the Monte Carlo estimator does not involve this additional term, we assume without loss of generality that $\mathcal{R}(\vtheta) = 0$. Extending the results to include  $\mathcal{R}(\vtheta)$ should be straightforward. 
\end{remark}

\subsubsection*{Strategy}
In general, we do not know how the matrices 
$\ma(\vtheta)$
depend on the parameters $\vtheta$, so bounding the
forward error
$\hat{\vtheta}-\vtheta^*$ is not feasible. Instead,
we bound the backward error $F(\hat{\vtheta})-F(\vtheta^*)$ of
 the computed $\hat{\vtheta}$.
Motivated by statistical learning theory 
\cite{KRV2024}, 
\cite[Section 8.4]{vershynin2018high}, we interpret this
backward error as an `excess risk'.
This also allows us to ignore uniqueness issues of $\hat{\vtheta}$ and $\vtheta^*$ in this general context.
Under specific assumptions \cite[Section 2.1]{Anitescu2012}, though, it is possible to show
local uniqueness.
Below we bound the excess risk in terms of the `overall loss function' $\max_{\vtheta\in\Theta}{|F(\vtheta)-\hat{F}(\vtheta)|}$.

\begin{remark}\label{r_excess}
Under the above assumptions
\begin{align*}
0\leq F(\hat{\vtheta})-F(\vtheta^*)\leq 
2\max_{\vtheta\in\Theta}{|F(\vtheta)-\hat{F}(\vtheta)|}.
\end{align*}
The lower bound follows from $\vtheta^*$ being a minimizer for $F$.
The upper bound follows from $\hat{\vtheta}$ being a minimizer for $\hat{F}$ \cite[Section 2.1]{KRV2024},
\begin{align*}
 F(\hat{\vtheta})-F(\vtheta^*)&\leq F(\hat{\vtheta})-\hat{F}(\hat{\vtheta})
 +\underbrace{\hat{F}(\hat{\vtheta})-\hat{F}(\vtheta^*)}_{\leq 0}+
 \hat{F}(\vtheta^*)-F(\vtheta^*)\\
 &\leq |F(\hat{\vtheta})-\hat{F}(\hat{\vtheta})|+|\hat{F}(\vtheta^*)-F(\vtheta^*)|
 \leq 2\max_{\vtheta\in\Theta}{|\hat{F}(\vtheta)-F(\vtheta)|}.
 \end{align*}
 \end{remark}

All that remains is to bound the overall loss
function $\max_{\vtheta\in\Theta}{|F(\vtheta)-\hat{F}(\vtheta)|}$.  To this end, we need the following 
ingredients. Analogous to  (\ref{e_abar}),
define the offdiagonal part 
of a matrix~$\ma(\vtheta)$ by
\begin{align}\label{e_offdiag}
\overline{\ma}(\vtheta)\equiv\ma(\vtheta)-\diag(\ma(\vtheta)).
\end{align}
The maximal norms of the offdiagonal part are
\begin{align}\label{e_offdiagn}
\alpha_{\xi}\equiv \max_{\vtheta\in\Theta}{\|\overline{\ma}(\vtheta)\|_{\xi}}, \qquad
 \xi \in \{2, F, M\},
\end{align}
where $\|\cdot\|_M$ is defined in (\ref{e_Mnorm}).

\subsection{Finite parameter space}\label{s_finite}
For a parameter space $\Theta$ with finite cardinality~$|\Theta|$, we determine two expressions for the sampling amount~$N$ so that the backward error
$F(\hat{\vtheta})-F(\vtheta^*)$ is bounded by $\epsilon$.

\begin{theorem}\label{t_4}
Given the assumptions in Section~\ref{s_ass},
let $0<\epsilon<1$ and $0<\delta<1$.
If
\begin{align*}
N \ge \frac{8}{\epsilon^2}\,\alpha_M^2\, \ln(|\Theta|/\delta), \quad \text{or} \quad 
N \ge \frac{16}{\epsilon^2}\, (2 \alpha_F^2 + \epsilon\alpha_2) \ln(2\,|\Theta|/\delta)
\end{align*}
then $0 \le F(\hat\vtheta) - F(\vtheta^*) < \epsilon$ with probability at least $1-\delta$.  
\end{theorem}

\begin{proof}
For the first bound, let $t>0$, and apply Lemma~\ref{l_1} and the union bound,
\begin{align*}
\myP[ \max_{\vtheta \in \Theta}| F(\vtheta) - \hat{F}(\vtheta)| \ge t] \le \sum_{\vtheta \in \Theta} \exp\left( \frac{-Nt^2}{2\,\|\overline\ma(\vtheta)\|_M^2}\right)  \le   
|\Theta|\exp\left( \frac{-Nt^2}{2\,\alpha_M^2}\right). 
    \end{align*}  
Set the upper bound equal to $\delta$ and $t = \epsilon/2$.
Solving for $N$ bounds the overall loss function by
$\max_{\vtheta \in \Theta}| F(\vtheta) - \hat{F}(\vtheta)| \leq\epsilon/2$
with
\begin{align*}
N \ge \frac{8}{\epsilon^2}\, \alpha_M^2 \,\ln(|\Theta|/\delta).
\end{align*}
At last, insert the above bound for the overall loss 
into Remark~\ref{r_excess}.

For the second bound, let $t>0$, and apply Lemma~\ref{l_2} and the union bound,
\begin{align*}
        \myP[ \max_{\vtheta \in \Theta}| F(\vtheta) - \hat{F}(\vtheta)| \ge t] &\le  \sum_{\vtheta \in \Theta} 2\,\exp\left( \frac{-Nt^2}{8(\|\overline\ma(\vtheta)\|_F^2 + t \|\overline\ma(\vtheta)\|_2)}\right)\\ 
        &\le 2\, |\Theta|\exp\left( \frac{-Nt^2}{8(\alpha_F^2  + t \alpha_2)}\right). 
    \end{align*}  
Now proceed as above:
Set the upper bound equal to $\delta$ and $t = \epsilon/2$.
Solving for $N$ bounds the overall loss function by
$\max_{\vtheta \in \Theta}| F(\vtheta) - \hat{F}(\vtheta)| \leq\epsilon/2$
with
\begin{align*}
N \ge \frac{16}{\epsilon^2}\, (2 \alpha_F^2 + \epsilon\alpha_2) \ln(2\,|\Theta|/\delta).
\end{align*}
At last, insert the above bound for the overall loss 
into Remark~\ref{r_excess}.
\end{proof}

 The sampling amount~$N$ in Theorem~\ref{t_4} 
 is small if the maximal offdiagonal mass is small and 
 the cardinality of the parameter space $\Theta$ is low.
 The bounds do not depend
on the matrix dimension~$m$.

The bounds for the OED application below end up depending
on the matrix dimension $m$,
through the cardinality of $\Theta$.

\begin{remark}
In the application to OED (Section~\ref{ssec:oed}),
where $k$ sensors are to be placed at $m>k$ candidate
locations, the cardinality of $\Theta$ can be bounded by
\cite[Exercise 0.0.5]{vershynin2018high}
\[ |\Theta| = \binom{m}{k} \le \sum_{j=0}^k \binom{m}{j} \le \left(\frac{em}{k} \right)^k, \]
which gives a sampling amount of
\[ N \ge  \frac{8}{\epsilon^2} \,\alpha_M^2\,
\left(k\ln\left(\frac{em}{k}\right)-\ln{(\delta)}\right)
\]
for the first bound, and
\[ N \ge  \frac{16}{\epsilon^2} \,
(\alpha_F^2+\epsilon\,\alpha_2)\,
\left(k\ln\left(\frac{em}{k}\right)+\ln{(2/\delta)}\right)\]
for the second bound.
\end{remark}

\subsection{Infinite Parameter Space}\label{s_infinite}
We determine the sampling amount~$N$ required to bound the backward error $F(\hat{\vtheta})-F(\vtheta^*)$ by $\epsilon$. The first bound (Theorem~\ref{t_1}) applies to general parameter spaces $\Theta$, 
while the second bound (Corollary~\ref{c_1}) is customized
to parameter spaces that are $K$-dimensional spheres.

We assume that the parameter space $\Theta$ is compact, and that the matrices
are Lipschitz continuous for some $L_2>0$ with
\begin{align}\label{e_theta1}
\|\ma(\vtheta)-\ma(\vtheta^{\prime})\|_2\leq L_2 
\|\vtheta-\vtheta^{\prime}\|_2\qquad 
\text{for all}\ \vtheta, \vtheta^{\prime}\in\Theta.
\end{align}

We discretize the infinite space $\Theta$ 
with a finite net $\mathcal{C}$.

\begin{definition}[Definitions 4.2.1 and 4.2.2 in \cite{vershynin2018high}]\label{d_net}
Let $\eta>0$. A subset $\mathcal{C}\subset \Theta$ is called an {\rm $\eta$-net} of 
$\Theta$ if every point of $\Theta$ lies within  a two-norm distance $\eta$ of some point
in~$\mathcal{C}$. 
That is, for all $\vtheta\in\Theta$ there exists a $\vtheta_0\in\mathcal{C}$ so that 
\begin{align*}
\|\vtheta-\vtheta_0\|_2\leq \eta.
\end{align*}
The smallest possible cardinality of the $\eta$-net 
$\mathcal{C}$ is called the {\rm covering number} 
$S(\eta)$ of~$\mathcal{C}$. This is illustrated in Figure~\ref{fig:epsilonnet}.
\end{definition}

The bounds in Theorem~\ref{t_1} below are the same as 
the ones for finite parameter spaces
in Theorem~\ref{t_4}, except that the cardinality
$|\Theta|$ is replaced by the covering number
$S(\eta)$ of the particular $\eta$-net.

\begin{figure}[!ht]
    \centering
    \includegraphics[width=0.5\linewidth]{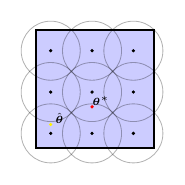}
    \caption{An illustration of the $\eta$-net argument. The parameter space $\Theta$ is covered by a net with 
    radius $\eta$. The covering number is $S(\eta) \le 9$ and the black dots represent the centers of the net.}
    \label{fig:epsilonnet}
\end{figure}

\begin{theorem}\label{t_1}
Given the assumptions above and in Section~\ref{s_ass},
let $0<\epsilon<1$, $0<\delta<1$, and abbreviate
\begin{align*}
\gamma\equiv S\left(\frac{\epsilon}{4mL_2}\right).
\end{align*}
If
\begin{align*}
N\geq \frac{8\alpha_M^2}{\epsilon^2}
\ln\left(\gamma/\delta\right)
\qquad \text{or} \qquad
N\geq \frac{16(2\alpha_F^2+\epsilon\alpha_2)}{\epsilon^2}\ln(2\gamma/\delta)
\end{align*}
then $0\le F(\hat{\vtheta})-F(\vtheta^*) \leq  \epsilon$ with probability at least $1-\delta$.
\end{theorem}

\begin{proof}
The proof relies on two subsequent lemmas in 
Section~\ref{s_auxiliary}:
Lemma~\ref{l_2b} bounds the loss 
$\max_{\vtheta\in\mathcal{C}}{|F(\vtheta)-\hat{F}(\vtheta)|}$ over  the net~$\mathcal{C}$, via 
Hoeffding's Theorem~\ref{t_hoeff}. Then
Lemma~\ref{l_3} extends this bound to the overall loss  $\max_{\vtheta\in\Theta}{|F(\vtheta)-\hat{F}(\vtheta)|}$. 

\begin{enumerate}
\item First lower bound for $N$:
Apply Lemma~\ref{l_3} and Remark~\ref{r_excess}
to conclude
  \begin{align*}
\myP\left[F(\hat{\vtheta})-F(\vtheta^*) \geq  4t\right]
 \leq S\left(\frac{t}{mL_2}\right)\exp\left(\frac{-2Nt^2}{\alpha_M^2}\right).
\end{align*}   
Setting  $\epsilon=4t$ gives
 \begin{align*}
\myP\left[F(\hat{\vtheta})-F(\vtheta^*) \geq  \epsilon\right]
 \leq \underbrace{S\left(\frac{\epsilon}{4mL_2}\right)}_{\gamma}\exp\left(\frac{-N\epsilon^2}{8\,\alpha_M^2}\right).
\end{align*}   
Set the upper above bound equal to $\delta$,
and solve for $N$.

\item Second lower bound for $N$:
Apply Lemma~\ref{l_3} and Remark~\ref{r_excess}
to conclude
  \begin{align*}
\myP\left[F(\hat{\vtheta})-F(\vtheta^*) \geq  4t\right]
 \leq 2\,S\left(\frac{t}{mL_2}\right)\exp\left(\frac{-Nt^2}{2(\alpha_F^2+2t\alpha_2)}\right).
 \end{align*}
Setting $\epsilon=4t$ gives 
\begin{align*}
\myP\left[F(\hat{\vtheta})-F(\vtheta^*) \geq  
\epsilon\right]\leq 2\,
\underbrace{S\left(\frac{\epsilon}{4mL_2}\right)}_{\gamma}
\exp\left(\frac{-N\epsilon^2}{16(2\alpha_F^2+\epsilon\alpha_2)}\right).
\end{align*}   
Set the above upper bound
equal to $\delta$, and solve for~$N$.
\end{enumerate}
\end{proof}

Theorem~\ref{t_1} implies that 
the sampling amount $N$ is
small if the matrices 
$\ma(\vtheta)$ have small offdiagonal mass.

Now we assume specifically that the parameter space $\mathcal{H}$ is a $K$-dimensional sphere that is centered at $\vtheta_c\in\mathcal{H}$,
\begin{align}\label{e_H}
\mathcal{H}\equiv \{\vtheta\in\real^K: \|\vtheta - \vtheta_c\|_2\leq B\}\qquad \text{for some}\quad B>0.
\end{align}
This allows us to replace the covering number
by an expression containing $K$ and~$B$.

\begin{corollary}\label{c_1}
Given the assumptions above and in Section~\ref{s_ass},
let $0<\epsilon<12mL_2B$ and $0<\delta<1$. If
\begin{align}\label{e_l44a}
N\geq \frac{8\alpha_M^2}{\epsilon^2}\left(K\ln(12mL_2B/\epsilon)-\ln(\delta)\right)
\end{align}
or
\begin{align}\label{e_l44b}
N\geq \frac{16(2\alpha_F^2+\epsilon\alpha_2)}{\epsilon^2}\left(K\ln(12mL_2B/\epsilon)+\ln(2/\delta)\right)
\end{align}
then $0\le F(\hat{\vtheta})-F(\vtheta^*) \leq  \epsilon$ with probability at least $1-\delta$.
\end{corollary}

\begin{proof}
Apply Lemma~\ref{l_net} with $\eta=\epsilon/(4mL_2)$,
\begin{align*}
\gamma=S\left(\frac{\epsilon}{4mL_2}\right)=
\max\left\{\left(\frac{12mL_2B}{\epsilon}\right)^K, 1\right\}.
\end{align*}
From $0<\epsilon<12mL_2B$ follows
\begin{align*}
\gamma=\left(\frac{12mL_2B}{\epsilon}\right)^K>1.
\end{align*}
Insert this into the expressions for $N$ in Theorem~\ref{t_1}.
\end{proof}

Corollary~\ref{c_1} implies that
the sampling amount~$N$ is
small, if the maximal offdiagonal mass is small and 
 the dimension~$K$ of the sphere $\mathcal{H}$ is low.
 The sampling amount increases only logarithmically
with the matrix dimension~$m$, radius $B$ of the sphere,
and the Lipschitz constant~$L_2$ of the matrices $\ma(\vtheta)$.

\subsection{Auxiliary lemmas}\label{s_auxiliary}
We present three auxiliary lemmas: two for the
proof of Theorem~\ref{t_1}, and one for the proof
of Corollary~\ref{c_1}.

The first lemma bounds
the loss function over the net in terms of the cardinality of the net and the sampling 
amount~$N$ from Hoeffding's Theorem~\ref{t_hoeff}.

 \begin{lemma}\label{l_2b}
 Under the assumptions in Sections \ref{s_ass} and~\ref{s_infinite}, let $t>0$. Then 
   \begin{align*}
\myP\left[\max_{\vtheta\in\mathcal{C}}{|\hat{F}(\vtheta)-F(\vtheta)}|\geq t\right]
\leq S(\eta)\exp\left(\frac{-Nt^2}{2\,\alpha_M^2}\right)
\end{align*}   
and
\begin{align*}
\myP\left[\max_{\vtheta\in\mathcal{C}}{|\hat{F}(\vtheta)-F(\vtheta)}|\geq t\right]
\leq 2\, S(\eta)\exp\left(\frac{-Nt^2}{8(\alpha_F^2+t\alpha_2)}\right).
\end{align*}   
 \end{lemma}
 
 \begin{proof}
  According to (\ref{e_unbiased}), $\hat{F}(\vtheta)$ is an unbiased estimator for $F(\vtheta)$.
  \begin{enumerate}
\item First bound: Lemma~\ref{l_1} implies for a fixed $\vtheta\in\mathcal{C}$ that
\begin{align*}
\myP\left[|\hat{F}(\vtheta)-F(\vtheta)|\geq  t\right]\leq
\exp\left(\frac{-Nt^2}{2\,\|\overline{\ma}(\vtheta)\|_M^2}\right).
\end{align*} 
The union bound over all elements in $\mathcal{C}$ gives
 \begin{align*}
\myP\left[\max_{\vtheta\in\mathcal{C}}{|\hat{F}(\vtheta)-F(\vtheta)}|\geq t\right]&\leq
\sum_{\vtheta\in\mathcal{C}}{\myP\left[|\hat{F}(\vtheta)-F(\vtheta)|\geq t\right]}\\
&\leq\sum_{\vtheta\in\mathcal{C}}{\exp\left(\frac{-Nt^2}{2\,\|\overline{\ma}(\vtheta)\|_M^2}\right)}
\leq
S(\eta)\exp\left(\frac{-Nt^2}{2\,\alpha_M^2}\right).
\end{align*} 

\item Second bound: Apply Lemma~\ref{l_2} and proceed as above.  
\end{enumerate}
 \end{proof}

 The second  lemma
extends the bound in Lemma~\ref{l_2b} to the whole parameter space~$\Theta$
 by exploiting the Lipschitz continuity 
 of~$\ma(\vtheta)$.
  
 \begin{lemma}\label{l_3}
  Under the assumptions in Sections \ref{s_ass} and~\ref{s_infinite}, let $t>0$. Then
  \begin{align*}
\myP\left[ \max_{\vtheta\in\Theta}{|\hat{F}(\vtheta)-F(\vtheta)|}
 \geq  2t\right]
 \leq S\left(\frac{t}{2mL_2}\right)\exp\left(\frac{-Nt^2}{2\,\alpha_M^2}\right)
\end{align*}   
and
 \begin{align*}
\myP\left[ \max_{\vtheta\in\Theta}{|\hat{F}(\vtheta)-F(\vtheta)|}
 \geq  2t\right]
 \leq 2\,S\left(\frac{t}{2mL_2}\right)\exp\left(\frac{-Nt^2}{8(\alpha_F^2+t\alpha_2)}\right).
\end{align*}   
    \end{lemma}
 
 \begin{proof}
Since $\Theta$ is compact and both $F$ and $\hat{F}$ are continuous,  there is a $\vtheta_h\in\Theta$ so that  
 \begin{align*}
 \max_{\vtheta\in\Theta}{|\hat{F}(\vtheta)-F(\vtheta)|}= |\hat{F}(\vtheta_h)-F(\vtheta_h)|.
\end{align*}
 Pick $\vtheta_0\in\mathcal{C}$ so that 
 $\|\vtheta_h-\vtheta_0\|_2\leq \eta$. The Lipschitz continuity~(\ref{e_theta1}) implies
 \begin{align*}
|F(\vtheta_h)-F(\vtheta_0)|&=|\sum_{i=1}^m{\ve_i\t(\ma(\vtheta_h)-\ma(\vtheta_0))\ve_i}|
\leq \sum_{i=1}^m{|\ve_i\t(\ma(\vtheta_h)-\ma(\vtheta_0))\ve_i|} \\
&\leq m\|\ma(\vtheta_h)-\ma(\vtheta_0)\|_2\leq mL_2\|\vtheta_h-\vtheta_0\|_2\leq mL_2\eta.
\end{align*}
 Similarly, the fact that the elements of Rademacher vectors are $\pm1$ implies for the
 estimator
  \begin{align*}
|\hat{F}(\vtheta_h)-\hat{F}(\vtheta_0)|&=
\left|\frac{1}{N}\sum_{i=1}^N{\vomega_i\t(\ma(\vtheta_h)-\ma(\vtheta_0))\vomega_i}\right|\\
&\leq \frac{1}{N}\sum_{i=1}^N{|\vomega_i\t(\ma(\vtheta_h)-\ma(\vtheta_0))\vomega_i|} \\
&\leq \frac{1}{N}\sum_{i=1}^N{\|\vomega_i\|_2^2\|\ma(\vtheta_h)-\ma(\vtheta_0)\|_2} \leq mL_2\eta.
%&\leq m\|\ma(\vtheta_h)-\ma(\vtheta_0)\|_2\leq mL\eta.
\end{align*}
Subtract and add $\hat{F}(\vtheta_0)$ and $F(\vtheta_0)$ and use the above
Lipschitz continuity,
\begin{align*}
|\hat{F}(\vtheta_h)-F(\vtheta_h)|&\leq 
\underbrace{|\hat{F}(\vtheta_h)-\hat{F}(\vtheta_0)|}_{\leq mL_2\eta}+
|\hat{F}(\vtheta_0)-F(\vtheta_0)|+
\underbrace{|F(\vtheta_0)-F(\vtheta_h)|}_{\leq mL_2\eta}\\
&=2mL_2\eta+|\hat{F}(\vtheta_0)-F(\vtheta_0)|\\
&\leq 2mL_2\eta+\max_{\vtheta\in\mathcal{C}}{|\hat{F}(\vtheta)-F(\vtheta)|}.
\end{align*}
To summarize
 \begin{align*}
 \max_{\vtheta\in\Theta}{|\hat{F}(\vtheta)-F(\vtheta)|}= |\hat{F}(\vtheta_h)-F(\vtheta_h)|
 \leq  2mL_2\eta+\max_{\vtheta\in\mathcal{C}}{|\hat{F}(\vtheta)-F(\vtheta)|}.
\end{align*}
Now apply Lemma~\ref{l_2b} to 
$\max_{\vtheta\in\mathcal{C}}{|\hat{F}(\vtheta)-F(\vtheta)|}$ to deduce for the overall loss
\begin{align*}
\myP\left[ \max_{\vtheta\in\Theta}{|\hat{F}(\vtheta)-F(\vtheta)|}
 \geq  2mL_2\eta+t\right]
 \leq S(\eta)\exp\left(\frac{-Nt^2}{2\,\alpha_M^2}\right)
\end{align*}   
and
\begin{align*}
\myP\left[ \max_{\vtheta\in\Theta}{|\hat{F}(\vtheta)-F(\vtheta)|}
 \geq  2mL_2\eta+t\right]
 \leq 2\,S(\eta)\exp\left(\frac{-Nt^2}{8(\alpha_F^2+t\alpha_2)}\right).
\end{align*}   
At last, set $\eta=\frac{t}{2mL_2}$.
\end{proof}

The third auxiliary lemma specifies 
the number of points in a net~$\mathcal{C}$ required to discretize the $K$-dimensional sphere~$\mathcal{H}$
in (\ref{e_H}).

\begin{lemma}[(4.10) in \cite{vershynin2018high}]\label{l_net}
Under the assumptions in Sections \ref{s_ass} and~\ref{s_infinite}, there exists an $\eta$-net 
$\mathcal{C}\subset\mathcal{H}$
whose number of elements is at most
\begin{align*}
S(\eta)\equiv
\max\left\{\left(\frac{3B}{\eta}\right)^K, 1\right\}.
\end{align*}
\end{lemma}

\section{Bounds based on generic chaining}\label{s_chaining}
The literature 
\cite{boucheron2003concentration,talagrand2005generic,
van2014probability,vershynin2018high} suggests that 
\textit{generic chaining} can produce bounds
that are superior to those derived from epsilon nets. Indeed, for Gaussian processes, Talagrand's majorizing measures theorem 
\cite[Theorem 8.6.1]{vershynin2018high}
shows that the bounds from chaining are optimal up to constants. 
However, we show that, in the context of trace optimization,
this superiority is not obvious. The bounds from
chaining do not require the Lipschitz continuity  of $\ma(\vtheta)$ but have the disadvantage of depending on functionals that are difficult to evaluate.

In contrast to the $\eta$-net in Definition~\ref{d_net}, which is often called an
epsilon-net or simply a net, chaining methods rely on multiscale nets. 
The idea is to set the context for chaining as follows: Represent the backward error
of the trace estimator for $\ma(\vtheta)$
as an element of a stochastic process 
on the parameter space $\Theta$,
\[ X_{\vtheta} = F(\vtheta) -\hat{F}(\vtheta) = \frac1N \sum_{j=1}^N(\trace(\ma(\vtheta)) - \vomega_j\t \ma(\vtheta) \vomega_j), \qquad \vtheta \in \Theta. \]
Since $\hat{F}(\vtheta)$ is an unbiased estimator for $F(\vtheta)$, the error
$X_{\vtheta}$ has zero mean.
To avoid issues with measurability of the supremum $\sup_{\vtheta \in \Theta} |X_{\vtheta}|$, we
assume that the process 
$\{X_{\vtheta}\}$ is separable \cite[Definition 5.22]{van2014probability}
-- an assumption that is almost always satisfied
\cite[Remark 5.23]{van2014probability}.

We determine the sampling amount~$N$ required to bound
the backward error of the estimator
for stochastic processes on
a parameter space $\Theta$. We derive 
bounds with subgaussian tails (Section~\ref{s_subg})
and mixed tails (Section~\ref{s_mixed}), as well as bounds 
specialized to parameter spaces that are spheres (Section~\ref{s_special}).
We end with several auxiliary results (Section~\ref{s_auxi}).
Where possible, we track the constants explicitly,
but make no claims as to their optimality. 

\subsection{Subgaussian tails}\label{s_subg}
We derive a bound for the backward error 
and the minimal sampling amount in terms of
a subgaussian tail.
After defining multi-scale nets
(Definition~\ref{d_adm}), we state our 
main result for the backward error (Theorem~\ref{thm:subgaussrisk}),
followed by a supporting result
(Theorem~\ref{thm:chainingsubgauss}).

\begin{definition}[(8.40) in \cite{vershynin2018high}]\label{d_adm}
 Given a set $\Theta$, a sequence of subsets $\{\mathcal{T}_k\}_{k \ge 0}$ is an admissible sequence if 
 \begin{enumerate}
\item  The subsets are increasing, i.e., $\mathcal{T}_k \subset \mathcal{T}_{k+1}$ for $k \ge 0$, and 
\item The subsets have bounded cardinality $|\mathcal{T}_k| \le 2^{2^k}$ for $k > 0$, and
$|\mathcal{T}_0| = 1$. 
\end{enumerate}
\end{definition}

For any  integer $\beta >0$,
we define the Talagrand $\gamma_\beta$-functional 
\cite[Definition 2.1]{Krahmer2014suprema},
\begin{align}\label{e_tala1}
\gamma_\beta(\Theta,d) \equiv \inf \sup_{\vtheta \in \Theta}\sum_{k=0}^\infty 2^{k/\beta}d(\vtheta,\mathcal{T}_k), 
\end{align}
where the infimum ranges over all admissible sequences, $\{\mathcal{T}_k\}_{k \ge 0}$.

We need the following (pseudo)-metrics\footnote{A
pseudo-metric does not necessarily satisfy the requirement that $d(\vtheta,\vphi) = 0$ implies $\vtheta = \vphi$.} on the offdiagonal
part (\ref{e_offdiag}) of matrices $\ma(\vtheta)\in\rmm$,
\begin{equation}\label{eqn:metrics} 
d_\xi(\vtheta,\vphi) \equiv  \|\overline\ma(\vtheta) - \overline\ma(\vphi)\|_\xi, \qquad \xi \in \{2, F, M\},  \end{equation} 
and, analogous to (\ref{e_offdiagn}), we define the corresponding norms 
\begin{equation}
    \label{eqn:alphas}
    \alpha_\xi \equiv \sup_{\vtheta \in \Theta} \|\overline\ma(\vtheta)\|_\xi, \qquad \xi \in \{2, F, M\},
\end{equation}
which are now based on a supremum.

Below, we present a subgaussian tail bound on the
backward error of the estimator, as well as a bound on the minimal
sampling amount~$N$ required to bound the 
backward error by~$\epsilon$.
\begin{theorem}\label{thm:subgaussrisk}
For $u \ge 2$
    \[ \myP\left[ {F}(\hat\vtheta) - F(\vtheta^*) \ge 
    \sqrt{\frac{8}{N}}\,u
    \left(2\gamma_2(\Theta,d_M) + \alpha_M \right)\right] \le 
    2\exp({-u^2/2}).   \]
Alternatively,     
let $ 0 < \delta < 1$ and $ \epsilon > 0$. If 
\begin{align}\label{e_t52b}
N \ge \frac{144}{ \epsilon^{2}} \ln(2/\delta) \max\{ \gamma_2(\Theta,d_M)^2, \alpha_M^2\},
\end{align}
then $ 0 \le {F}(\hat\vtheta) - F(\vtheta^*) \le \epsilon$ with probability at least $1-\delta$. 
     \end{theorem}
     
\begin{proof}
From \cref{r_excess} and the triangle
inequality follows
\begin{equation}\label{eqn:inter1}  0 \le F(\hat{\vtheta})  - F(\vtheta^*) \le 2\sup_{\vtheta \in \Theta}|X_{\vtheta}|
%\[ \sup_{\vtheta \in \Theta}|X_{\vtheta}| 
\le 2\,\left( \sup_{\vtheta \in \Theta} | X_{\vtheta} - X_{\vtheta_0}| + |X_{\vtheta_0}|\right)
\end{equation}
Bounding the first summand with  \cref{thm:chainingsubgauss} shows that for $u>2$, 
\begin{align}\label{e_inter2}
\myP\left[\sup_{\vtheta\in \Theta}| X_{\vtheta} - X_{\vtheta_0} |  \ge 
\sqrt{\frac{2}{N}}\,2\,u\,\gamma_2(\Theta,d_M)\right] 
\leq \exp(-u^2/2).
\end{align}

Apply \cref{l_1} to the second summand
\begin{align*}
\myP[|X_{\vtheta_0}| \ge t] \le 
\exp\left(-\frac{Nt^2}{2\,\alpha_M^2}\right), \end{align*}
and set $t = \sqrt{\frac{2}{N}}\alpha_M\,u$
for $u>0$,
\begin{align}\label{e_inter3}
\myP\left[|X_{\vtheta_0}| \ge \sqrt{\frac{2}{N}} \alpha_M\, u\right] \le \exp(-u^2)\leq
\exp(-u^2/2).
\end{align}
Adding (\ref{e_inter2}) and (\ref{e_inter3}) and applying the union bound shows that with probability at least $1 - 2\exp(-u^2/2)$ 
\begin{align*}
\sup_{\vtheta \in \Theta} | X_{\vtheta} - X_{\vtheta_0}| + |X_{\vtheta_0}|\leq 
\sqrt{\frac{2\,}{N}}\,u\,
\left(2\gamma_2(\Theta,d_M) + \alpha_M\right).
\end{align*}
Inserting this into (\ref{eqn:inter1})
gives the probability tail bound.

Next, set $\delta=2\exp(-u^2/2)$ so that $u = \sqrt{2\ln(2/\delta)}$. Then with probability at least $1-\delta$, 
\[ 0 \le F(\hat{\vtheta})  - F(\vtheta^*) \le \sqrt{\frac{2}{N}}\,2u\,
\left(2\gamma_2(\Theta,d_M) + \alpha_M\right)
 \le  \epsilon, \] 
where the penultimate inequality follows 
from the choice of~$N$ in (\ref{e_t52b}).
\end{proof}

We compare the bound on the sampling
amount from generic chaining to that from the
$\eta$-net.

\begin{remark}[Comparison of Theorems \ref{t_1}
and~\ref{thm:subgaussrisk}]\label{r_com2}
A comparison of the two bounds for the minimal sampling amount~$N$ is difficult and inconclusive.

The bound in Theorem~\ref{t_1} is
\begin{align}\label{e_com3}
N\geq \frac{8\alpha_M^2}{\epsilon^2}
\ln\left(\gamma/\delta\right), \qquad\text{where}\quad
\gamma\equiv S\left(\frac{\epsilon}{4mL_2}\right)
\end{align}
and $S(\cdot)$ denotes the covering number of the net.
The bound in Theorem~\ref{thm:subgaussrisk} is
\begin{align}\label{e_com4}
N \ge \frac{144}{ \epsilon^{2}} \ln(2/\delta) \max\{ \gamma_2(\Theta,d_M)^2, \alpha_M^2\},
\end{align}
where $\gamma_2(\Theta,d_M)$ is the Talagrand functional~(\ref{e_tala1}).

\begin{enumerate}
\item Since $\hat{F}$ is a Monte Carlo estimator,
both bounds grow with $-\ln(\delta)/\epsilon^{-2}$.
That is, the sampling amount increases as the desired
error $\epsilon$ and the failure probability~$\delta$ become smaller. 

\item Both bounds depend on the ``size'' of $\Theta$,
either through the covering number or the Talagrand
functional. However, (\ref{e_com4}) with the Talagrand functional is much more difficult to evaluate.

\item The chaining bound (\ref{e_com4}) has the advantage of not requiring Lipschitz continuity of~$\ma(\vtheta)$, while (\ref{e_com3}) does.

\item The chaining bound (\ref{e_com4}) 
might grow less with the offdiagonal 
mass $\alpha_M$ and the ``size'' of $\Theta$
because it depends on the maximum
of $\alpha_M^2$ and $\gamma_2(\Theta,d_M)^2$, while 
(\ref{e_com3})
depends on the product $\alpha_M^2\ln(\gamma)$.

But then again, the chaining bound (\ref{e_com4})
depends more
strongly on the ``size'' of $\Theta$ than (\ref{e_com3}), which depends only
on the logarithm of the covering number.
\end{enumerate}
\end{remark}

Next is the foundation for the proof of
Theorem~\ref{thm:subgaussrisk}. 

\begin{figure}[!ht]
    \centering
    \includegraphics[width=0.5\linewidth]{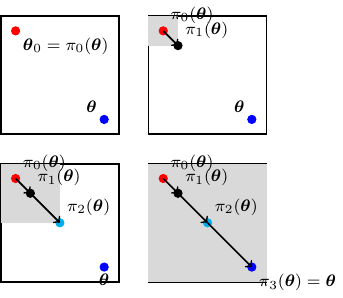}
    \caption{An illustration of the chaining argument. We construct the chain $\vtheta_0 = \pi_0(\vtheta) \rightarrow \pi_1(\vtheta) \rightarrow \pi_2(\vtheta) \rightarrow \pi_3(\vtheta) = \vtheta.$ The shaded regions represent the sets $\mathcal{T}_1, \mathcal{T}_2$, and $\mathcal{T}_3 = \Theta$, with $\mathcal{T}_0 = \{\vtheta_0\}$.}
    \label{fig:chaining}
\end{figure}

\begin{theorem}\label{thm:chainingsubgauss}
Let $\vtheta_0 \in \Theta$ and $u\geq 2$. With probability at least $1 -\exp(-u^2/2)$,
    \[  \sup_{\vtheta\in \Theta}| X_{\vtheta} - X_{\vtheta_0} |  \le \sqrt{\frac{8}{N}}\, u\, \gamma_2(\Theta,d_M).   \] 
\end{theorem}

\begin{proof}
The separability assumption allows us to limit
consideration to the case where $\Theta$ is a finite set. An example for extensions to the infinite case in
a more general context is
given in \cite[Theorem 5.24]{van2014probability}.

The idea is to break down $\sup_{\vtheta\in\Theta}{|X_{\vtheta} - X_{\vtheta_0}|}$ into a sequence of
suprema over smaller subsets $\mathcal{T}_i$ from an admissible sequence.
To this end,
the proof proceeds in several steps, along the lines
of \cite[pages 10-11]{talagrand2005generic}.
\begin{enumerate}
\item \textit{Concentration inequality.}
    Since the elements $X_{\vtheta}$ have zero mean, \cref{l_1} implies that they have  subgaussian increments, 
    \begin{align*}
\myP[|X_{\vtheta} - X_{\vphi}| \ge t ]  \le \exp\left(-\frac{Nt^2}{2\,d_M(\vtheta,\vphi)^2}\right). 
\end{align*}
Below we use the following version:
For $u>0$, integers $i\geq 1$, and change of variable
$t=u\,\sqrt{2}\, 2^{i/2}d_M(\vtheta,\vphi)/\sqrt{N}$ the above
is equal to 
   \begin{align}\label{e_c1}
\myP\left[|X_{\vtheta} - X_{\vphi}| \ge u \,\sqrt{2}\,2^{i/2}d_M(\vtheta,\vphi)/\sqrt{N}\right]  \le \exp\left(-u^2 2^{i}\right). 
\end{align}

 \item \textit{Constructing a chain from $\vtheta_0$ to another parameter.} Let $\vtheta\in\Theta$, let $\{\mathcal{T}_i\}$ with $\mathcal{T}_0=\{\vtheta_0\}$ be an admissible sequence, and let $\pi_i(\vtheta)$
    be an element closest to~$\vtheta$
    from the subset $\mathcal{T}_i$, that is,
    \begin{align*}
\pi_{i}(\vtheta) \in \argmin_{\vtheta' \in \mathcal{T}_i} d_M(\vtheta,\vtheta'),\qquad i\geq 0.
\end{align*}

Since the subsets $\mathcal{T}_i$ are increasing
in size, the elements $\pi_i(\vtheta)$ define a
chain from $\vtheta_0$ 
to $\vtheta$,
    \[  \vtheta_0 =\pi_0(\vtheta) \rightarrow \pi_{1}(\vtheta) \rightarrow \cdots \rightarrow \pi_T(\vtheta) = \vtheta, \]
    where $T$ is finite because $\Theta$ is finite. This is illustrated in Figure~\ref{fig:chaining}.

We can decompose the distance between the processes
    $X_{\vtheta}$ and $X_{\vtheta_0}$ as the telescoping sum of successive increments,
    \[ X_{\vtheta} - X_{\vtheta_0} = \sum_{i=1}^T (X_{\pi_i(\vtheta)} - X_{\pi_{i-1}(\vtheta)}).\] 

Furthermore, because $\pi_i(\vtheta)$ is an element in $\mathcal{T}_i$ closest to $\vtheta$, the distance
between the parameters $\vtheta$ and $\pi_i(\vtheta)$ is equal to the
distance between $\vtheta$ and the set $\mathcal{T}_i$,
\begin{align}\label{e_inter6}
d_M(\vtheta, \pi_i(\vtheta)) = d_M(\vtheta,\mathcal{T}_i)\equiv \inf_{\vphi\in\mathcal{T}_i}{d_M(\vtheta,\vphi)},\qquad 0 \le i \le T.
\end{align}

\item \textit{Bounding the distance between successive processes in the chain.}
We define events that represent the 
    distance between successive processes,
    \[  E_{\vtheta,i}  \equiv \left\{ |X_{\pi_i(\vtheta)} - X_{\pi_{i-1}(\vtheta)}| \ge u \sqrt{2}\,2^{i/2}\frac{\sigma_i}{\sqrt{N}}\right\} \qquad \vtheta \in \Theta, \quad 1 \le i \le T, \] 
 where   $\sigma_i \equiv d_M(\pi_i(\vtheta),\pi_{i-1}(\vtheta))$ is the distance between the
 corresponding parameters
in the chain. Then (\ref{e_c1}) implies
\begin{align}\label{e_inter5}
\myP[E_{\vtheta,i} ] \le \exp\left(-u^2 2^{i}\right), \qquad \vtheta \in \Theta, \> 1 \le i \le T.
\end{align}
    
\item \textit{Union bound over all possible process distances.} 
 Since the elements in the chain are not unique, we need to sum over all possible events, that is, over all possible 
parameter pairs.
The number of possible pairs
$\{\pi_i(\vtheta),\pi_{i-1}(\vtheta)\}$ is bounded by
\begin{align*}
|\mathcal{T}_i||\mathcal{T}_{i-1}| 
\le |\mathcal{T}_i|^2 \le 2^{2^{i+1}},\qquad 
1\leq i\leq T.
\end{align*} 

The probability of all events occurring is
\begin{align}\label{e_53prob}
\myP\left[\bigcup_{\vtheta\in \Theta} \bigcup_{1 \le i \le T}  E_{\vtheta,i} \right]  &\le   \sum_{i=1}^T |\mathcal{T}_i|^2  \myP[ E_{\vtheta,i} ] \nonumber\\
&\leq \sum_{i=0}^\infty 2^{2^{i+1}}\exp\left(-u^2\,2^{i}\right) \le \exp(-u^2/2),
\end{align}
where the second inequality follows from~(\ref{e_inter5}), and the third from Lemma~\ref{lemma:sum} in Section~\ref{s_auxi} for $u\geq 2$.

The expression~(\ref{e_53prob}) is the probability 
that  $\sup_{\vtheta \in \Theta}|X_{\vtheta} - X_{\vtheta_0}|$ is not bounded. 
Now we derive the
actual bound for this supremum by adding the
bounds in the individual $E_{\vtheta,i}$,
$1\leq i\leq T$.

\item \textit{Bounding $\sup_{\vtheta \in \Theta}|X_{\vtheta} - X_{\vtheta_0}|$.}
From~(\ref{e_53prob}) and the definition of
$E_{\vtheta,i}$
follows that with probability  at least  $1-\exp(-u^2/2)$,
\begin{align*} \sup_{\vtheta \in \Theta}|X_{\vtheta} - X_{\vtheta_0}| 
\le  &\> \sup_{\vtheta \in \Theta} \sum_{i=1}^T |  X_{\pi_i(\vtheta)} - X_{\pi_{i-1}(\vtheta)}| 
    \le   \sup_{\vtheta \in \Theta}u \sum_{i=1}^T \sqrt{2}\, 2^{i/2}\frac{\sigma_i}{\sqrt{N}} \\
    \le &  \sup_{\vtheta \in \Theta} 
    \sqrt{\frac{2}{N}}\,u \sum_{i=1}^T  2^{i/2} (d_M(\pi_i(\vtheta), \vtheta) + d_M(\vtheta,\pi_{i-1}(\vtheta))) \\
    \leq & \sup_{\vtheta \in \Theta}
    \sqrt{\frac{2}{N}}\, 2u\sum_{i=0}^T 2^{i/2} d_M(\vtheta, \mathcal{T}_i)\\
    \le &\sup_{\vtheta \in \Theta} 
    \sqrt{\frac{8}{N}}\, u\,\sum_{i=0}^\infty 2^{i/2} d_M(\vtheta, \mathcal{T}_i). \end{align*} 
Here, the third inequality follows from the triangle inequality,
\[ \sigma_i = d_M(\pi_i(\vtheta),\pi_{i-1}(\vtheta)) \le d_M(\pi_i(\vtheta), \vtheta) + d_M(\vtheta,\pi_{i-1}(\vtheta)), \qquad 1 \le i \le T, \]
and the fourth inequality from (\ref{e_inter6}). We have shown that with probability  at least  $1-\exp(-u^2/2)$,
\begin{align*} 
\sup_{\vtheta \in \Theta}|X_{\vtheta} - X_{\vtheta_0}| 
\leq \sup_{\vtheta \in \Theta} 
\sqrt{\frac{8}{N}}\,u \sum_{i=0}^\infty 2^{i/2} d_M(\vtheta, \mathcal{T}_i).
\end{align*}
Taking the infimum over all admissible sequences $\{\mathcal{T}_i\}_{i \ge 0}$ in the 
upper bound implies that with probability  at least $1-\exp(-u^2/2)$, 
\begin{align*}
\sup_{\vtheta \in \Theta}|X_{\vtheta} - X_{\vtheta_0}| \le 
\sqrt{\frac{8}{N}}\,u\,
\gamma_2(\Theta,d_M).
\end{align*}
\end{enumerate}
\end{proof}

\subsection{Mixed tails}\label{s_mixed}
We derive a bound for the backward error 
and minimal sampling amount in terms of a 
mixed tail.
After defining another version of multi-scale nets
(Definition~\ref{d_adm2}), we state our 
main result for the backward error (Theorem~\ref{thm:chainingmixed}),
and make a comparison with the 
corresponding $\eta$-net
bound (Remark~\ref{r_com2}),
followed by a supporting result
(Theorem~\ref{thm:mixedtails}).

For the mixed tails,
we need a more stringent version of nets.
Recall that the $\eta$-nets in 
Definition~\ref{d_adm} contain, at every level $k$, a single set $\mathcal{T}_k$ which is a proper subset of $\Theta$ (except possibly at the last level). In contrast, the partitions defined below contain, at every
level $k$, a collection of sets that cover
all of $\Theta$. This is necessary because the 
proof of Theorem~\ref{thm:mixedtails} requires intersections of sets associated with different tails,
and some of them have to be non-empty. This would not be possible with a
single set at every level.

\begin{definition}[Definition 1.2.3 in \cite{talagrand2005generic}]\label{d_adm2}
Given a set $\Theta \in \real^K$, a partition of 
$\Theta$ is a set~$\mathcal{A}$ whose 
elements are disjoint subsets of $\Theta$ and whose union
equals $\Theta$. That is,
$\mathcal{A}=\{A_i\}$ where $A_i\subset\Theta$, $A_i\cap A_j=\emptyset$
for $i\neq j$, and $\bigcup_i{A_i}=\Theta$.

A sequence of partitions
$\{\mathcal{A}_k\}_{k\geq 0}$ of $\Theta$ is increasing if every element of $\mathcal{A}_{k+1}$ is a subset of an element of $\mathcal{A}_k$. 

An increasing sequence of partitions $\{\mathcal{A}_k\}_{k\geq 0}$ of $\Theta$ is
admissible if $|\mathcal{A}_k| \le 2^{2^k} $ for $k>0$ and
$|\mathcal{A}_0|=1$.
This is illustrated in Figure~\ref{fig:admissible}. 
\end{definition}

\begin{figure}[!ht]
    \centering
    \includegraphics[width=0.75\linewidth]{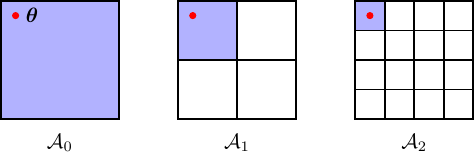}
    \caption{An illustration of the sequence of partitions $\{\mathcal{A}_k\}_{0 \le k\leq 2}$. The red circle denotes the point $\vtheta$, and the blue squares  denote the unique elements $A_0(\vtheta), A_1(\vtheta)$, and $A_2(\vtheta)$.}
    \label{fig:admissible}
\end{figure}

The diameter of a set $ \mathcal{T} \subset \Theta$ associated with a (pseudo-)metric $d_\xi$ is 
\begin{equation}
\Delta_{\xi} (\mathcal{T}) \equiv \sup_{\vtheta,\vphi \in \mathcal{T}} d_\xi(\vtheta,\vphi), \qquad \xi \in \{2, F, M\}.
\end{equation}

Let $\{\mathcal{A}_k\}_{k\ge 0}$ be an admissible sequence of 
partitions of $\Theta$, and let
$A_k(\vtheta)$ be the unique element in   $\mathcal{A}_k$ containing $\vtheta$.  
For any integer $\alpha>0$, define the second Talagrand functional 
\begin{equation}\label{e_tala2}
\gamma_\alpha'(\Theta, d_{\xi}) \equiv \inf \sup_{\vtheta \in \Theta} \sum_{k=0}^{\infty} 2^{k/\alpha} \Delta_{\xi}(A_k(\vtheta)),
\end{equation}
where the infimum ranges over all admissible sequences of partitions $\{\mathcal{A}_k\}_{k\geq 0}$.

The Talagrand functionals (\ref{e_tala1}) and 
(\ref{e_tala2}) are equivalent, in the sense that for any (pseudo-)metric $d$ and integer $\beta>0$
\begin{align}\label{e_gamma}
\gamma_\beta(\Theta,d) \le \gamma'_{\beta}(\Theta,d) %
\leq C_{\beta}\gamma_\beta(\Theta,d),   
\end{align}
where $C_\beta$ is a constant that depends only on~$\beta$.

\begin{remark}
The functional $\gamma_\alpha'(\Theta, d_{\xi})$
is an upper bound on the diameter of $\Theta$. That is,
\begin{align*}
\gamma_\alpha'(\Theta, d_{\xi}) \ge \Delta_{\xi}(\Theta).
\end{align*}
This follows from $A_0(\vtheta) = \Theta$  and the fact
that regardless of the partition
$\{\mathcal{A}_k\}_{k \ge 0}$, 
\begin{align*}
\sup_{\vtheta \in \Theta} \sum_{k=0}^{\infty} 2^{k/\alpha} \Delta_{\xi}(A_k(\vtheta)) \ge 
\sup_{\vtheta \in \Theta}\Delta_{\xi}(A_0(\vtheta)) = \Delta_{\xi}(\Theta).
\end{align*}
\end{remark}

Below we present a mixed tail bound on the
backward error of the estimator, 
as well as a bound on the minimal
sampling amount~$N$ required to bound the error by~$\epsilon$.

\begin{theorem}\label{thm:chainingmixed}
For $u\geq 4$,
with probability at least $1-4\exp(-u/2)$ 
\[  0 \le F(\hat\vtheta) - F(\vtheta^*) \le \frac{32u}{N} \, \gamma_1'(\Theta, d_2) + \sqrt{\frac{128u}{N}} \,\gamma_2'(\Theta, d_F) +  \frac{8u}{N}\,\alpha_2 + \sqrt{\frac{16u}{N}}\,\alpha_F. \] 
Alternatively, let $0<\delta \leq 4\exp(-2)$ and $ \epsilon > 0$. If 
\begin{align}\label{e_t55a}
N   \ge 4096\epsilon^{-2}\ln(4/\delta) \max\left\{ \epsilon\gamma_1'(\Theta,d_2), \epsilon\alpha_2  , (\gamma_2'(\Theta, d_F))^2, \alpha_F^2     \right\},  
\end{align}
then $ 0 \le {F}(\hat\vtheta) - F(\vtheta^*) \le \epsilon$ with probability at least $1-\delta$.
\end{theorem}

\begin{proof}
\cref{l_2} implies for a fixed $\vtheta_0 \in \Theta$, 
\[\myP[| X_{\vtheta_0}| \ge t] \le 2\exp\left(-\frac{Nt^2}{8(\alpha_F^2 + t\alpha_2)}\right). \]
Analogous to step~3 in the proof of the subsequent Theorem~\ref{thm:mixedtails}, a change of variable
gives for $u > 0$
\[ \myP\left[| X_{\vtheta_0}| \ge \frac{4u}{N}\alpha_2 + \sqrt{\frac{4u}{N}} \alpha_F \right] 
\le 2 \exp(-u/2).\]
From (\ref{eqn:inter1}), the union bound and Theorem~\ref{thm:mixedtails}
follows that for $u\geq 4$ with probability at least $1 - 4\exp(-u/2)$,  
\[ \sup_{\vtheta \in \Theta} |X_{\vtheta}| \le \frac{16u}{N} \, \gamma_1'(d_2) + \sqrt{\frac{32u}{N}} \,\gamma_2'(d_F) +  \frac{4u}{N}\,\alpha_2 + \sqrt{\frac{4u}{N}}\, \alpha_F,\] 
and 
\[  0 \le F(\hat\vtheta) - F(\vtheta^*) \le \frac{32u}{N} \, \gamma_1'(d_2) + \sqrt{\frac{128u}{N}} \,\gamma_2'(d_F) +  \frac{8u}{N}\,\alpha_2 + \sqrt{\frac{16u}{N}}\,\alpha_F.  \] 
Set $\delta=4\exp(-u/2)$ so that $u = 2\ln(4/\delta)$ for 
$0<\delta \leq 4\exp(-2)$. The expression for~$N$
in~(\ref{e_t55a}) is obtained by setting each of
the 4 summands above to $\epsilon/4$, solving for $N$, and taking the maximum.
\end{proof}

A comparison of the sampling amounts from 
Theorems~\ref{thm:chainingmixed} 
and~\ref{t_1} is similar to but as inconclusive
as the one in Remark~\ref{r_com2}.

The following 
represents the basis for the proof of Theorem~\ref{thm:chainingmixed}.

\begin{theorem}\label{thm:mixedtails}
 Let $\vtheta_0 \in \Theta$ and $u\geq 4$.
 With probability at least $1-2\exp(-u/2)$ 
\[ \begin{aligned} \sup_{\vtheta \in \Theta} |X_{\vtheta}-X_{\vtheta_0}| \le & \>  \frac{16u}{N} \, \gamma_1'(d_2) + \sqrt{\frac{32u}{N}} \,\gamma_2'(d_F).   \end{aligned}\]
\end{theorem}

\begin{proof}
The proof essentially follows that of~\cite[Section 1.2]{talagrand2005generic}, but it pays attention to the constants and relies on the mixed tail bound in \cref{l_2}. As in the proof of Theorem~\ref{thm:chainingsubgauss},
the separability assumption allows us to limit consideration 
to the case where $\Theta$ is a finite set. 

\begin{enumerate}
\item \textit{Constructing admissible sequences
of partitions.}  
Since
\cref{l_2} involves two different (pseudo-)metrics $d_2$ and $d_F$, the chaining argument in \cref{thm:chainingsubgauss} needs to be expanded to two different admissible sequences of partitions, one for each (pseudo-)metric, followed by the formation
of their intersection \cite[Proof of Theorem~1.2.7]{talagrand2005generic}.

We construct these two admissible sequences 
of partitions $\{\mathcal{B}_k\}_{k\ge 0}$ and $\{\mathcal{C}_k\}_{k\ge 0}$ with
$\mathcal{B}_0=\Theta=\mathcal{C}_0$ so that they 
are close to their respective lower bounds $\gamma_1^{\prime}(\Theta, d_2)$ and
$\gamma_2^{\prime}(\Theta, d_F)$,
\begin{align}
\sum_{k=0}^\infty 2^{k} \Delta_2(B_k(\vtheta)) &\le 2 \gamma_1'(\Theta, d_2),\qquad 
\vtheta\in\Theta,\label{e_b}\\
\sum_{k=0}^\infty 2^{k/2} \Delta_F(C_k(\vtheta)) &\le 2 \gamma_2'(\Theta,d_F),\qquad
\vtheta\in\Theta,\label{e_c}
\end{align}
where, as in~\eqref{e_tala2}, the sets $B_k(\vtheta)$ and $C_k(\vtheta)$ are the unique elements in $\mathcal{B}_k$ and $\mathcal{C}_k$, respectively, that contain $\vtheta$. 

Then we combine $\{\mathcal{B}_k\}_{k\ge 0}$ and $\{\mathcal{C}_k\}_{k\ge 0}$
to construct a new increasing sequence of partitions
of~$\Theta$,
\begin{align*}
\mathcal{A}_0 \equiv \Theta,\qquad
\mathcal{A}_k = \{ B\cap C | B \in \mathcal{B}_{k-1}, C \in \mathcal{C}_{k-1}\}, \qquad k>0.
\end{align*}
Each set $\mathcal{A}_k$ of intersections is non-empty because the elements of $\mathcal{B}_{k-1}$ and 
$\mathcal{C}_{k-1}$ cover all of $\Theta.$
The sequence of partitions $\{\mathcal{A}_k\}_{k \ge 0}$ is admissible with cardinality $|\mathcal{A}_0|=1$ and
\begin{align*}
|\mathcal{A}_k| \le |\mathcal{B}_{k-1}| |\,\mathcal{C}_{k-1}|\le \left(2^{2^{k-1}}\right)^2 = 2^{2^k},\qquad k > 0.
\end{align*}
This is not the tightest bound on the cardinality,
but what is important is the upper bound of the form $2^{2^k}$.

\item \textit{Defining the chain.}
For the admissible sequence of partitions $\{\mathcal{A}\}_{k\geq 0}$ we construct sets 
$T_k \subset \Theta$ by selecting exactly one element from each set of $\mathcal{A}_k$ for $k>0$, 
starting with $T_0 = \{\vtheta_0\}$. 
The sets have cardinality  $|T_0|= 1$ and $|T_k| \le 2^{2^k}$ for $k > 0.$ 

Pick a parameter~$\vtheta$, and define  $\pi_k(\vtheta)$ as the element in the singleton set
$T_k \cap A_k(\vtheta)=\{\pi_k(\vtheta)\}$ for $k \ge 0$. 
Define a chain from $\vtheta_0$  to $\vtheta$,
    \[  \vtheta_0 =\pi_0(\vtheta) \rightarrow \pi_{1}(\vtheta) \rightarrow \cdots \rightarrow \pi_M(\vtheta) = \vtheta, \]
    where $M$ is finite because $\Theta$ is finite.

As before, we decompose the distance between the processes
    $X_{\vtheta}$ and $X_{\vtheta_0}$ as the telescoping sum of increments,
\[ X_{\vtheta} - X_{\vtheta_0} = \sum_{k=1}^M (X_{\pi_k(\vtheta)}- X_{\pi_{k-1}(\vtheta)}). \]

We need the following bounds for the increments. As in~\eqref{e_tala2}, let $A_k(\vtheta)$ be the unique element of $\mathcal{A}_k$ containing $\vtheta$
for $k\geq 0$.
The nestedness of the sets in the partition implies
that
$\pi_k(\vtheta),\pi_{k-1}(\vtheta) \in A_{k-1}(\vtheta) \subset B_{k-2}(\vtheta)$, hence
\begin{align*}
d_2(\pi_1(\vtheta), \pi_0(\vtheta))
\leq \Delta_2(B_0(\vtheta)), \qquad
d_2(\pi_k(\vtheta),\pi_{k-1}(\vtheta)) \le \Delta_2(B_{k-2}(\vtheta))
\end{align*}
for $2\leq k\leq M$.
This together with (\ref{e_b}) implies
\begin{align}\label{e_bb}
\sum_{k=1}^M2^k d_2(\pi_k(\vtheta),\pi_{k-1}(\vtheta)) \le \sum_{k\ge 0}2^k \Delta_2(B_k(\vtheta))\le 2\gamma_1'(\Theta,d_2). 
\end{align}
Similarly, (\ref{e_c})  implies
\begin{align}\label{e_cc}
\sum_{k=1}^M2^{k/2} d_F(\pi_k(\vtheta),\pi_{k-1}(\vtheta)) \le  2\gamma_2'(\Theta,d_F).
\end{align}

\item \textit{Re-expressing the concentration inequality.}
Lemma~\ref{l_2} implies 
\begin{align*}
\myP[ |X_{\vtheta} - X_{\vphi}| \ge t ] \le & \>  2 \exp\left(\frac{-Nt^2}{8 (d_F(\vtheta,\vphi)^2 +t d_2(\vtheta,\vphi)) }\right).
\end{align*}
Set $\zeta_F^2 = 8d_F(\vtheta,\vphi)^2/N$ and $\zeta_2 = 8d_2(\vtheta,\vphi)/N$. In $u = t^2/(\zeta_F^2 + t\zeta_2)$ solve for~$t$ by taking the positive root 
\[ t = \frac{u\zeta_2}{2} + \sqrt{\frac{u^2}{4} \zeta_2^2  + u \zeta_F^2} 
\le u\,\zeta_2 + \sqrt{u}\,\zeta_F .\]
This gives for $u > 0$
\[ \myP\left[ |X_{\vtheta} - X_{\vphi}| \ge \frac{8}{N}d_2(\vtheta,\vphi)  \, u + \sqrt{\frac{8}{N}}d_F(\vtheta,\vphi) \,\sqrt{u}\right] 
\le 2 \exp(-u).\] 

\item \textit{Applying the union bound.}
Motivated by (\ref{e_bb}) and (\ref{e_cc}) we can define the event for $u>0$,
\begin{equation*}
\begin{split}
\Omega_{k,\vtheta, u}&\equiv \biggl\{
|X_{\pi_k(\vtheta)} - X_{\pi_{k-1}(\vtheta)}|  \ge    \> \frac{8u}{N}\,2^k \,d_2(\pi_k(\vtheta),\pi_{k-1}(\vtheta))     \\ 
& \qquad +  \sqrt{\frac{8u}{N}}\,2^{k/2}\, d_F(\pi_k(\vtheta),\pi_{k-1}(\vtheta)) \biggr\}, \quad 0\leq k\leq M.
\end{split}
\end{equation*} 
The previous step implies $\myP[\Omega_{k,\vtheta, u}] \le 2\exp(-2^k \, u).$ For every $\vtheta \in \Theta$, and fixed $k$,  there are $|T_k||T_{k-1}| \le 2^{2^{k+1}}$ different possible pairs $(\pi_k(\vtheta),\pi_{k-1}(\vtheta))$. 
A simple union bound gives for $u\geq 4$\begin{align*}
\myP\left[ \cup_{\vtheta \in \Theta }  \cup_{0 \le k \le M}  \Omega_{k,\vtheta,u}\right]
\le & \> 2 \sum_{k=0}^M |T_k||T_{k-1}|\exp(-2^k \, u)\\
\le & \> 2\sum_{k=0}^M  2^{2^{k+1}}
\exp(-2^k\, u) \\
\le & 2\sum_{k=0}^\infty   2^{2^{k+1}}
\exp(-2^k \, u) \le 2\exp(-u/2),
\end{align*} 
where the last inequality follows from
\cref{lemma:sum} in Section~\ref{s_auxi}.

Combining this with steps 2 and 3 shows that with probability at least $1-2\exp(-u/2)$ for $u\ge 4 $
\[ \begin{aligned} \sup_{\vtheta \in \Theta} |X_{\vtheta}-X_{\vtheta_0}| 
\le & \>   \sum_{k=1}^M\frac{8u}{N}\,2^k \,d_2(\pi_k(\vtheta),\pi_{k-1}(\vtheta))     +  \sqrt{\frac{8u}{N}}\,2^{k/2}\, d_F(\pi_k(\vtheta),\pi_{k-1}(\vtheta))  \\  
\le  &\> \frac{8u}{N} \,  2\gamma_1'(d_2) + \sqrt{\frac{8u}{N}} \, 2\gamma_2'(d_F)=
\frac{16 u}{N} \,\gamma_1'(d_2) + \sqrt{\frac{32u}{N}} \,\gamma_2'(d_F),\end{aligned} \] 
where the last inequality follows from (\ref{e_bb}) and~(\ref{e_cc}). 
\end{enumerate}
\end{proof}

\subsection{Spherical parameter spaces}\label{s_special}
Working with the Talagrand functionals $\gamma_\beta$ and $\gamma_\beta'$ can be difficult in practice. 
It is possible, however, in certain instances
when the functionals can be expressed
in terms of explicitly known covering numbers,
as in the case of spherical parameter spaces
(Lemma~\ref{l_net}).

Below, we derive bounds on the backward error
and sampling amounts for parameter spaces that are $K$-dimensional spheres (Corollaries \ref{c_54a} and~\ref{c_54b}), and make a comparison with the 
corresponding $\eta$-net
bound (Remark~\ref{r_com}).

We extend the definition of $\eta$-net in 
Definition~\ref{d_net} to general metrics and nets that are not necessarily subsets of $\mathcal{T}$.

\begin{definition}[Definition 5.5 in \cite{van2014probability}]
Let $\eta > 0$, and $(\mathcal{T},d)$ with $\mathcal{T}\subset \Theta$ be a metric space.
A set $\mathcal{C}$ is called an $\eta$-net of~$(\mathcal{T}, d)$ if, 
    for every $\vtheta \in \mathcal{T}$, there exists a point $\vtheta' \in \mathcal{C}$ with $d(\vtheta,\vtheta') \le \eta$. The smallest cardinality of an $\eta$-net for $(\mathcal{T},d)$ is called the covering number $S(\mathcal{T},d,\eta)$. 
\end{definition}

The $\eta$-nets enjoy the following monotonicity property. 

\begin{remark}\label{r_mono}
If for two (pseudo-)metrics $d$ and $d'$ on a set
$\mathcal{T}$
there is a constant~$C>0$ so that 
\begin{align}\label{e_r53}
d(\vtheta,\vphi) \le C d'(\vtheta,\vphi)\qquad
\text{for all}\quad \vtheta,\vphi\in\mathcal{T}
\end{align}
then
\begin{align*}
S(\mathcal{T},d,\eta) \le S(\mathcal{T}, d', \eta/C).
\end{align*}
\end{remark}

\begin{proof}
Let  $\mathcal{C}$ be a $(\eta/C)$-net for 
$(\mathcal{T},d^{\prime})$
with covering number $S(\mathcal{T}, d', \eta/C)$.
That is,
for all $\vtheta\in\mathcal{T}$ there exists a
$\vphi\in\mathcal{C}$ with
\begin{align*}
d^{\prime}(\vtheta,\vphi)\leq \eta/C.
\end{align*}
Then (\ref{e_r53}) and the above imply that 
for all $\vtheta\in\mathcal{T}$ there exists a 
$\vphi\in\mathcal{C}$ so that 
\begin{align*}
d(\vtheta,\vphi)\leq C\, d^{\prime}(\vtheta,\vphi)\leq C(\eta/C)=\eta.
\end{align*}
Thus, $\mathcal{C}$ is an $\eta$-net for $(\mathcal{T},d)$, and the minimality of the covering number
implies
\[ S(\mathcal{T},d,\eta) \le |\mathcal{C}| = S(\mathcal{T}, d', \eta/C). \] 
\end{proof}

Example~\ref{e_mono_sp} in Section~\ref{s_auxi} presents 
special cases of Remark~\ref{r_mono} that are 
required for the proofs of Corollaries \ref{c_54a}
and~\ref{c_54b}.

We convert Talagrand's functionals into Dudley integral inequality \cite[Theorem 8.1.3]{vershynin2018high} which lets us
bound $\gamma_\beta(\Theta,d)$ in terms of the covering numbers $S(\Theta,d,\eta)$ 
{of infinitely many $\eta$-nets},
\begin{align}\label{e_dudley}
\gamma_\beta(\Theta,d) \le C_\beta \int_{0}^{\infty} \left(\ln S(\Theta,d, \eta) \right)^{1/\beta} d\eta,
\end{align}
where $C_{\beta}$ is a constant that depends only on $\beta$. Note that the conversion to the Dudley integral form results in a loss of optimality~\cite[Exercise 8.2]{vershynin2018high}, but this integral form is easier to work with, in practice. %\AKS{Note to self: add a sentence or two on Dudley's integral form.}

We derive two bounds for the case where,
as in (\ref{e_H}), the parameter space~$\mathcal{H}$
is a $K$-dimensional sphere that is centered at $\vtheta_c\in\mathcal{H}$,
\begin{align*}
\mathcal{H}\equiv \{\vtheta\in\real^K: \|\vtheta - \vtheta_c\|_2\leq B\}\qquad \text{for some}\quad B>0.
\end{align*}

\begin{corollary}\label{c_54a}
Let $\ma(\vtheta)$ be Lipschitz continuous, 
\begin{align*}
\|\ma(\vtheta) - \ma(\vphi)\|_M \le L_M \|\vtheta-\vphi\|_2\qquad \text{for all}\ \vtheta, \vphi \in  \mathcal{H}.
\end{align*}
If
\[ N \ge C\,
\frac{\ln(2/\delta)}{\epsilon^2}\,\max\{\alpha_M^2, K(BL_M)^2\} \] 
for an absolute constant~$C>0$, then 
$ 0 \le \hat{F}(\hat\vtheta) - F(\vtheta^*) \le \epsilon$ with probability at least~$1-\delta$.
\end{corollary}

\begin{proof}
In \cref{thm:subgaussrisk}, we only need to bound $\gamma_2(\Theta,d_M)$.  
From the Lipschitz continuity of $\ma(\vtheta)$, (\ref{e_m1}) in Section~\ref{s_auxi},
and Lemma~\ref{l_net} follows
     \[S(\Theta,d_M,\eta) \le S(\Theta, d_2', \eta/L_M) \le  \max\left\{ \left(\frac{3BL_M}{\eta}\right)^K , 1\right\}. \]
Convert this to the Dudley integral form (\ref{e_dudley})
    \[ \begin{aligned}\gamma_2(\Theta,d_M) \lesssim & \>  \int_0^\infty \sqrt{\ln S(\Theta,d_M,\eta)}d\eta \le  \int_0^\infty \max\left\{ \sqrt{K\ln \left(\frac{3BL_M}{\eta}\right)}, 0 \right\} d\eta \\ 
    = & \> \int_0^a  \sqrt{K\ln \left(\frac{3BL_M}{\eta}\right)} d\eta, \end{aligned}\]
where $a = 3BL_M$. The maximum is zero for $\eta>a$ because the logarithm is negative then.
Simplify the integral with \cref{lemma:integral}, 
    \[ \gamma_2(\Theta,d_M) \lesssim 
    \int_0^a \sqrt{K \ln (a/\eta)} d\eta  
        \le a \sqrt{K}\Gamma(3/2)
    \lesssim \sqrt{K}BL_M.
    \] 
With the expression in \cref{thm:subgaussrisk},
the lower bound for $N$ is now
\begin{align*}
N \ge C\, \frac{\ln(2/\delta)}{\epsilon^2}\, \max\{K(BL_M)^2, \alpha_M^2\}
\ge C \,\frac{\ln(2/\delta)}{\epsilon^2} \,\max\{ \gamma_2(\Theta,d_M)^2, \alpha_M^2\}.
\end{align*}
\end{proof}

\begin{remark}[Comparison of Corollaries \ref{c_54a}
and~\ref{c_1}]\label{r_com}
A comparison of the two bounds for the sampling
amount~$N$ is difficult and inconclusive.

The bound in (\ref{e_l44a}) in Corollary~\ref{c_1} is
\begin{align}\label{e_com1}
N\geq \frac{8\alpha_M^2}{\epsilon^2}\left(K\ln(12mL_2B/\epsilon)-\ln(\delta)\right)
\end{align}
while the one in Corollary~\ref{c_54a} is
\begin{align}\label{e_com2}
N \ge C\,
\frac{\ln(2)-\ln(\delta)}{\epsilon^2}\,\max\{\alpha_M^2, K(BL_M)^2\}.
\end{align}
\begin{enumerate}
\item Since $\hat{F}$ is a Monte Carlo estimator,
both bounds grow with $-\ln(\delta)/\epsilon^{-2}$.
That is, the sampling amount increases as the desired
error $\epsilon$ and the failure probability~$\delta$ become smaller. 

\item Although both bounds require $\ma(\vtheta)$ to be
Lipschitz continuous, the Lipschitz constants $L_M$
and $L_F$
belong to different norms for $\ma(\vtheta)$.

\item Clearly, (\ref{e_com1}) has the advantage of all
constants being explicitly specified, while the constant~$C$ in (\ref{e_com2}) is not known.

\item Barring the unknown constant $C$, (\ref{e_com2}) 
does not depend explicitly on the matrix dimension~$m$,
while (\ref{e_com1}) depends on $m$ weakly through 
the logarithm.

\item The chaining bound (\ref{e_com2}) 
might grow less with the offdiagonal 
mass $\alpha_M$ and the parameter
space dimension~$K$ because it depends on the maximum
of 
$\alpha_M^2$ and $K(BL_M)^2$, while 
(\ref{e_com1})
depends on the product $\alpha_M^2K\ln(L_2B)$.

But then again, the chaining bound depends more
strongly on the radius $B$ and the Lipschitz 
constant than (\ref{e_com1}).
\end{enumerate}
\end{remark}

Below is the mixed tail version of Corollary~\ref{c_54a}.

\begin{corollary}\label{c_54b}
Let $\ma(\vtheta)$ be Lipschitz continuous, 
$$\|\ma(\vtheta) - \ma(\vphi)\|_F \le L_F \|\vtheta-\vphi\|_2 \quad \text{and} \quad \|\ma(\vtheta) - \ma(\vphi)\|_2 \le L_2 \|\vtheta-\vphi\|_2$$ for all $\vtheta, \vphi \in \mathcal{H}$. 
If
\[ N \ge C \,\frac{\ln(2/\delta)}{\epsilon^2}\,\max\{\alpha_F^2, \epsilon \alpha_2, K(BL_F)^2, \epsilon KBL_2 \} \] 
for an absolute constant $C>0$, then
$ 0 \le \hat{F}(\hat\vtheta) - F(\vtheta^*) \le \epsilon$ with probability at least~$1-\delta$.
\end{corollary}

\begin{proof}
    The proof is similar to that of Corollary~\ref{c_54a}, except we use Theorem~\ref{thm:chainingmixed},
where we need to bound $\gamma_1^{\prime}(\Theta,d_2)$ and $\gamma_2^{\prime}(\Theta,d_F)$. 

As for $\gamma_2^{\prime}(\Theta, d_F)$, the Lipschitz continuity of $\ma(\vtheta)$, 
(\ref{e_m2}) in Section~\ref{s_auxi},
and Lemma~\ref{l_net} imply
     \[S(\Theta,d_F,\eta) \le S(\Theta, d_2', \eta/L_F) \le  \max\left\{ \left(\frac{3BL_F}{\eta}\right)^K , 1\right\}. \]
With (\ref{e_gamma}), (\ref{e_dudley}) and the above, we bound as in the proof of Corollary~\ref{c_54a}
\begin{align*}
\gamma_2'(\Theta,d_F) &\lesssim \gamma_2(\Theta,d_F) \lesssim \int_0^\infty \sqrt{\ln S(\Theta,d_F,\eta)}d\eta\\
&\leq
\int_0^\infty \max\left\{ \sqrt{K\ln \left(\frac{3BL_F}{\eta}\right)}, 0 \right\}\, d\eta\lesssim
    \sqrt{K}BL_F. 
    \end{align*}
    As for $\gamma_1^{\prime}(\Theta,d_2)$, the Lipschitz continuity of $\ma(\vtheta)$, (\ref{e_m3}) in Section~\ref{s_auxi},
and Lemma~\ref{l_net} imply
     \[S(\Theta,d_2,\eta) \le S(\Theta, d_2', \eta/(2L_2)) \le  \max\left\{ \left(\frac{6BL_2}{\eta}\right)^K , 1\right\}. \]
Again, with (\ref{e_gamma}), (\ref{e_dudley}) and the above, we bound as in the proof of Corollary~\ref{c_54a}
\begin{align*}
\gamma_1'(\Theta,d_2) &\lesssim \gamma_1(\Theta,d_2) \lesssim \int_0^\infty \sqrt{\ln S(\Theta,d_2,\eta)}d\eta
\leq\int_0^\infty \max\left\{ \sqrt{K\ln \left(\frac{6BL_2}{\eta}\right)}, 0 \right\}\, d\eta\\
&=\int_0^a \sqrt{K\ln \left(\frac{6BL_2}{\eta}\right)}\, d\eta,
    \end{align*}
where $a = 6BL_2$. 
The maximum is zero for $\eta>a$ because the logarithm is negative then. Simplify the integral with \cref{lemma:integral}, 
    \[ \gamma_1^{\prime}(\Theta,d_2)\lesssim \gamma_1(\Theta,d_2) 
    \lesssim 
    \int_0^a K \ln (a/\eta) d\eta  =a\,K\Gamma(2)
    \lesssim KBL_2.
    \] 
 Then proceed as in the proof of Corollary~\ref{c_54a}.
\end{proof}

A comparison of the sampling amounts from 
Corollary~\ref{c_54b} and (\ref{e_l44b}) in Corollary~\ref{c_1} is similar to but as inconclusive
as the one in Remark~\ref{r_com}.

\subsection{Auxiliary results}\label{s_auxi}
We present a bound for an infinite sum 
(Lemma~\ref{lemma:sum}), the simplification 
of an integral (Lemma~\ref{lemma:integral}),
 a bound on the offdiagonal norm of a matrix
(Lemma~\ref{l_normdiag}), and examples of covering
numbers for different types of Lipschitz
continuity (Example~\ref{e_mono_sp}).

The bound below is used in Theorem~\ref{thm:mixedtails} for the case $p=1$.

\begin{lemma}\label{lemma:sum} Let $ p\ge 1$ be an integer and $u\geq 4^{1/p}$. Then
\[  \sum_{k=0}^\infty  2^{2^{k+1}}\exp(-2^ku^p) \le \exp(-u^p/2).\]
\end{lemma}

\begin{proof}
Pull the power of two into the exponential,
    \begin{align*}
2^{2^{k+1}}\exp(-2^ku^p) &= \exp(2^k(2\ln 2 - u^p))\\
&= \exp\left(2^k\left(2\ln{2} -\frac{u^p}{2}\right)-2^k\frac{u^p}{2}\right)\\
&= \exp\left(2^k\left(2\ln{2} -\frac{u^p}{2}\right)\right)
\exp\left(-2^k\frac{u^p}{2}\right)\\
& \le \exp\left(2^k\left(2\ln{2} -\frac{u^p}{2}\right)\right)
\exp\left(-\frac{u^p}{2}\right).
\end{align*}
Since $2\ln 2 < 2$ and, by assumption,
$u^p \ge  4$ 
    \[ \begin{aligned}\sum_{k=0}^\infty 2^{2^{k+1}}\exp(-2^ku^p) \le  & \> \exp(-u^p/2)\sum_{k=0}^\infty \exp(2^k(2\ln 2 - u^p/2 )) \\
    \le & \> \exp(-u^p/2)\sum_{k=0}^\infty \exp(2^k(2\ln 2 - 2)). \end{aligned}\] 
    The last inequality follows since $u^p/2 \ge 2$ and $2^k \ge 1$ for $k \ge 0$. Finally, we can bound the sum as
    \[ \sum_{k=0}^\infty \exp(2^k(2\ln 2 - 2)) \le \sum_{k=0}^\infty \exp(2k(2\ln 2 - 2)) < 1,  \]
    by summing the geometric series. 
\end{proof}

The simplification below is used in 
Corollary~\ref{c_54b}.

\begin{lemma}\label{lemma:integral}
Let $a>0$, $x>0$, and $p > 0$ be an integer. Then 
\[ \int_{0}^a (\ln(a/x) ) ^{1/p}dx = 
\Gamma\left(\frac{1}{p} + 1\right)a, \]
where $\Gamma(x) \equiv \int_0^\infty t^{x-1} e^{-t}dt$ is the Gamma function.
\end{lemma}
\begin{proof}
Change the integration variable to $t = \ln(a/x) $ so that $x = ae^{-t}$. Differentiating
gives $dx = -a e^{-t}dt$ and 
    \[  \int_{0}^a (\ln(a/x) ) ^{1/p}dx = \int_0^\infty t^{1/p}a e^{-t}\,dt = a \,\Gamma\left(\frac{1}{p} + 1\right).\] 
\end{proof}

The following bound on the offdiagonal norm of a 
matrix is used in Example~\ref{e_mono_sp}.

\begin{lemma}[Section 1 in \cite{bhatia1989comparing}]
\label{l_normdiag}
Let $\mb\in\rnn$ and $\overline{\mb}\equiv\mb-\diag(\mb)$ its offdiagonal part. Then
 \[ \|\overline\mb \|_2 \le  2\| \mb\|_2.    \] 
    \end{lemma}

Below are special cases of Remark~\ref{r_mono},
where we bound the covering numbers
for three different types of Lipschitz continuity
required in the proofs of Corollaries \ref{c_54a} and~\ref{c_54b}.

\begin{example}\label{e_mono_sp}
Let $\ma(\vtheta)$ be Lipschitz continuous with
\begin{align*}
d_M(\vtheta,\vphi) \leq L_M \|\vtheta-\vphi\|_2\qquad \text{for all}\quad \vtheta, \vphi\in\Theta.
\end{align*}
Then 
\begin{align*}
d_M(\vtheta,\vphi) =  \|\overline\ma(\vtheta) - \overline\ma(\vphi)  \|_M \le \| \ma(\vtheta) - \ma(\vphi)\|_M  \le  L_M \|\vtheta-\vphi\|_2, 
\end{align*}
 where the first inequality follows from     
$\|\overline\mb\|_M =\sum_{i\ne j} |b_{ij}| \le \sum_{i,j} |b_{ij}|= \|\mb\|_M$. 
Hence
\begin{align*}
d_M(\vtheta,\vphi)\leq L_M \, d_2^{\prime}(\vtheta,\vphi)\qquad \text{with}\quad
d_2^{\prime}(\vtheta,\vphi)\equiv \|\vtheta-\vphi\|_2,
\end{align*}
and
\begin{align}\label{e_m1}
S(\mathcal{T},d_M,\eta) \le S(\mathcal{T}, d_2^{\prime}, \eta/L_M).
\end{align}

Let $\ma(\vtheta)$ be Lipschitz continuous with
\begin{align*}
d_F(\vtheta,\vphi) \leq L_F \|\vtheta-\vphi\|_2\qquad \text{for all}\quad \vtheta, \vphi\in\Theta.
\end{align*}
Then 
\begin{align*}
d_F(\vtheta,\vphi) =  \|\overline\ma(\vtheta) - \overline\ma(\vphi)  \|_F \le \| \ma(\vtheta) - \ma(\vphi)\|_F  \le  L_F \|\vtheta-\vphi\|_2, 
\end{align*}
and, as above,
\begin{align}\label{e_m2}
S(\mathcal{T},d_F,\eta) \le S(\mathcal{T}, d_2^{\prime}, \eta/L_F).
\end{align}
Let $\ma(\vtheta)$ be Lipschitz continuous with
\begin{align*}
d_2(\vtheta,\vphi) \leq L_2 \|\vtheta-\vphi\|_2\qquad \text{for all}\quad \vtheta, \vphi\in\Theta.
\end{align*}
Then
 \begin{align*}
d_2(\vtheta,\vphi) =  \|\overline\ma(\vtheta) - \overline\ma(\vphi)  \|_2 \le 2\,\| \ma(\vtheta) - \ma(\vphi)\|_2  \le  2 \,L_2 \|\vtheta-\vphi\|_2,
\end{align*}
 where the first inequality follows from Lemma~\ref{l_normdiag}.
Hence
\begin{align*}
d_2(\vtheta,\vphi)\leq 2\,L_2 \, d_2^{\prime}(\vtheta,\vphi),
\end{align*}
and
\begin{align}\label{e_m3}
S(\mathcal{T},d_2,\eta) \le S(\mathcal{T}, d_2^{\prime}, \eta/(2\,L_2)).
\end{align}

\end{example}

\section{Conclusion and discussion}\label{sec:conc}
We have extended the randomized trace estimation of a fixed 
matrix 
to trace optimization of parameter dependent matrices,
and derived bounds for the required sampling amounts.
The bounds based on nets (Section~\ref{s_main}) are straightforward to verify and apply, and come with constants fully specified. 
In contrast, the bounds based on chaining (Section~\ref{s_chaining}), although optimal in some contexts, are more difficult to apply and do not necessarily have a clear advantage over net-based arguments. Our results contribute theoretical justifications in the Monte Carlo trace estimators 
for SAA optimization techniques in \cite{chi2019going,chung2024}. 

A drawback of our chaining argument is the
dependence of the bounds (Theorem~\ref{thm:mixedtails}) on Talagrand's $\gamma_1'$ functional. This dependence
could have possibly been avoided by bounding instead
the suprema of chaos processes, which 
are intimately connected to trace estimation.
Arguments 
based on chaining with decoupling can produce bounds
that involve only the $\gamma_2$ functional
\cite{dirksen2015tail,Krahmer2014suprema}.

One avenue for future work could be practical bounds for Talagrand functionals in special cases. 
There are several options. The Dudley inequality bounds 
Talagrand functionals in terms
of covering numbers  
(Section~\ref{s_special}). The covering numbers, in turn, can be bounded if the matrix has low rank~\cite{candes2011tight}, while a
more general technique relies on Sudakov minorization
\cite[Section 7.4]{vershynin2018high}.
Talagrand functionals can also be bounded in optimization problems with sparsity promoting solutions~\cite{Krahmer2014suprema,vershynin2009role}.

Finally,  moving beyond nets, 
it is worth exploring the potential of another approach
from statistical learning in trace optimization, namely
that of Rademacher 
complexity~\cite[Section 4.5]{bach2024learning}.

\bibliographystyle{siamplain}
\bibliography{refs}

\end{document}